\documentclass{article}

\usepackage[preprint]{neurips_2025}

\usepackage[utf8]{inputenc} 
\usepackage[T1]{fontenc}    
\usepackage{hyperref}       
\usepackage{url}            
\usepackage{booktabs}       
\usepackage{amsfonts}       
\usepackage{nicefrac}       
\usepackage{microtype}      
\usepackage{xcolor}         
\usepackage{wrapfig}
\usepackage{graphicx}
\usepackage{amsmath, amsthm}
\usepackage{rotating} 
\usepackage{multirow} 
\usepackage{float}  
\usepackage{algorithm}
\usepackage{algpseudocode}
\usepackage[normalem]{ulem}
\useunder{\uline}{\ul}{}
\newtheorem{theorem}{Theorem}
\newtheorem{lemma}{Lemma}  

\usepackage[most]{tcolorbox}
\usepackage{xcolor}
\newtcbox{\myshadowbox}[1][]{colframe=black!75!white, colback=yellow!10, boxrule=1pt, arc=4mm, auto outer arc, drop shadow, #1}

\title{Reinforcement Speculative Decoding for Fast Ranking}

\author{%
  Du Yingpeng \\
  College of Computing and Data Science\\
  Nanyang Technological University\\
  Singapore\\
  \texttt{dyp1993@pku.edu.cn} \\
  \And
  Wei Tianjun\thanks{Corresponding Author} \\
  College of Computing and Data Science\\
  Nanyang Technological University\\
  Singapore\\
  \texttt{tjwei2-c@my.cityu.edu.hk} \\
  \AND
  Sun Zhu \\
  Information Systems Technology and Design \\
  Singapore University of Technology and Design \\
  Singapore \\
  \texttt{zhu\_sun@sutd.edu.sg} \\
  \And
  Zhang jie \\
  College of Computing and Data Science\\
  Nanyang Technological University\\
  Singapore\\
  \texttt{zhangj@ntu.edu.sg} \\
}

\begin{document}


\maketitle

\begin{abstract}
Large Language Models (LLMs) have been widely adopted in ranking systems such as information retrieval (IR) systems and recommender systems (RSs). To alleviate the latency of auto-regressive decoding, some studies explore the single (first) token decoding for ranking approximation, but they suffer from severe degradation in tail positions. Although speculative decoding (SD) methods can be a remedy with verification at different positions, they face challenges in ranking systems due to their left-to-right decoding paradigm. Firstly, ranking systems require strict latency constraints, but verification rounds in SD methods remain agnostic; Secondly, SD methods usually discard listwise ranking knowledge about unaccepted items in previous rounds, hindering future multi-token prediction, especially when candidate tokens are the unaccepted items. In this paper, we propose a Reinforcement Speculative Decoding method for fast ranking inference of LLMs.  
To meet the ranking systems' latency requirement, we propose an up-to-down decoding paradigm that employs an agent to iteratively modify the ranking sequence under a constrained budget. Specifically, we design a ranking-tailored policy optimization, actively \textbf{\textit{exploring}} optimal multi-round ranking modification policy verified by LLMs via reinforcement learning (RL). To better approximate the target LLM under the constrained budget, we trigger the agent fully \textbf{\textit{utilizing}} the listwise ranking knowledge about all items verified by LLMs across different rounds in RL, enhancing the modification policy of the agent. More importantly, we demonstrate the theoretical robustness and advantages of our paradigm and implementation. Experiments on both IR and RS tasks show the effectiveness of our proposed method.

\end{abstract}


\section{Introduction}\label{sec:intro}
Recently, Large Language Models (LLMs) have demonstrated extensive knowledge and reasoning abilities across various domains. Ranking systems, notably information retrieval (IR) systems and recommender systems (RSs), have experienced significant advancements due to the integration of LLMs \cite{wang2025re2llm,pradeep2023rankvicuna,zhu2023large,du2024enhancing}. Typically, ranking systems aim to output a ranking (sequence) of candidate passages based on their ``relevance'' to a given query. For example, the query in RSs can refer to a specific user with her behaviors, so RSs aim to rank the candidate items based on the user's preferences \cite{du2025quasi}. Despite their remarkable effectiveness by integrating LLMs, a significant barrier remains: the inherent latency due to the auto-regressive decoding scheme of LLMs severely limits the practical applicability of LLM-based ranking systems. To mitigate this issue, current approaches primarily adopt two decoding strategies: Single Token Decoding (STD) and Speculative Decoding (SD).


STD-based methods \cite{reddy2024first,pradeep2023Rankzephyr,pradeep2023rankvicuna} attempt to solely utilize the output logits of the first generated identifier (aka single LLM encoding) to approximate a ranking of candidates, where a higher logit indicates a higher ranking score of the item. Specifically, FIRST \cite{reddy2024first} and Rankzephyr \cite{pradeep2023Rankzephyr} propose to conduct reranking-based supervised fine-tuning (SFT) based on the correctness of relative passage orders. 
\begin{wrapfigure}{r}{0.55\linewidth}
    \vspace{-5pt}  
    \centering
    \includegraphics[width=\linewidth]{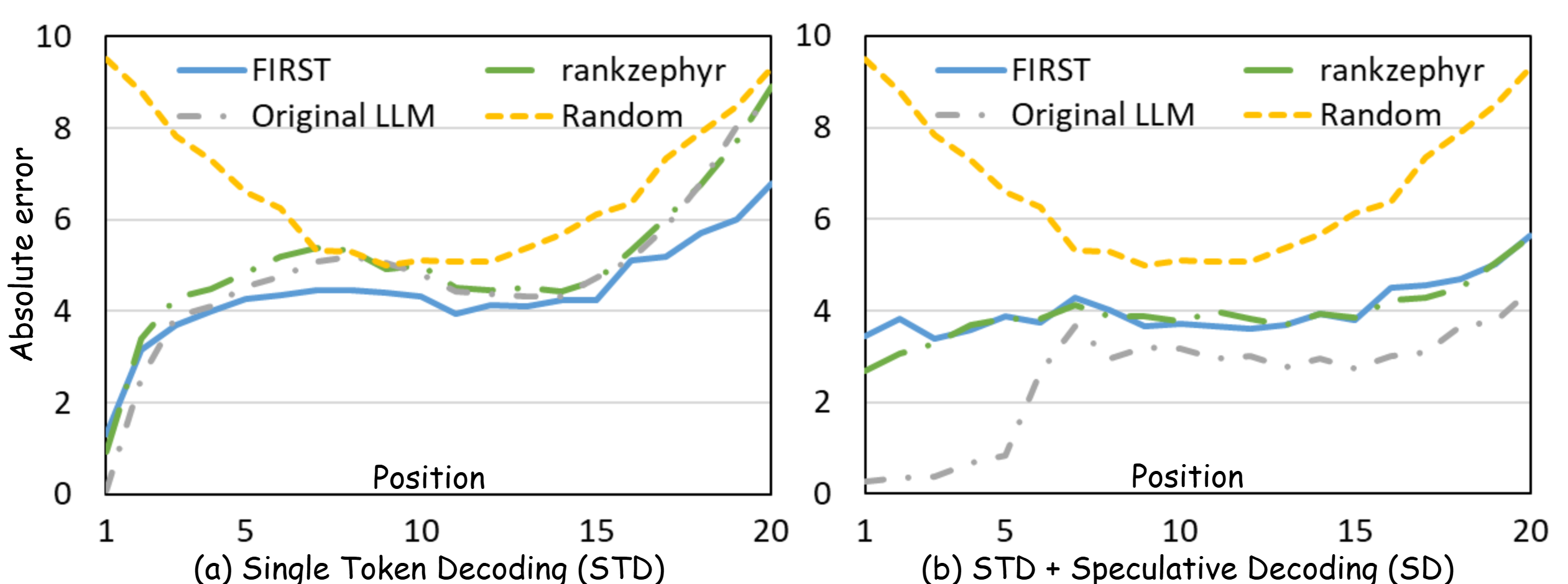}
    \vspace{-15pt}  
    \caption{{\small Comparison among different methods and decoding strategies on MS MARCO dataset with Llama3.2-3B-Instruct, where SD requires 5×LLM encoding.}} 
    \label{fig:1}
    \vspace{-15pt}  
\end{wrapfigure}
\autoref{fig:1} (a) shows the position-aware absolute error of the ranking indicated by FIRST \cite{reddy2024first}, Rankzephyr \cite{pradeep2023Rankzephyr}, the original LLM with STD, and random ranking \cite{nguyen2016ms}. 
It indicates that the STD strategy can effectively approximate the top-half items, showing its potential to reduce LLM inference latency for ranking tasks. However, their error becomes unacceptably large in the lower positions, indicating that these methods suffer from severe degradation in the tail rankings. Worse still, as shown in \autoref{fig:1} (b), STD methods (e.g., FIRST and Rankzephyr) can hardly be enhanced by SD strategy \cite{leviathan2023fast}. This is attributed to STD methods rely on fine-tuning LLMs, which leads to distribution drift and results in inaccurate verification during SD. To tackle these problems, we propose a multi-round modification (decoding) strategy without fine-tuning LLMs. By verifying and modifying rankings at different positions, we can bridge between the predicted ranking and the target ranking (i.e., auto-regressive \underline{\textit{ranking}} generated by the \underline{target} LLM).

SD methods can fasten text generation while preserving consistency at different positions (e.g., tail position), indicated by line of ``original LLM'' in \autoref{fig:1} (b). Specifically, they divide the inference process into iterations of a low-cost draft stage and a parallelized verification stage, where each verification allows accept multiple consistent tokens over the draft in a single LLM encoding. However, they are not one-size-fits-all solutions for LLM-based ranking systems due to their left-to-right decoding paradigm, as shown in \autoref{intro_2} (a).
\textit{\textbf{First, ranking systems usually have strict latency requirements for response.}} Specifically, ranking systems such as RSs and IR systems typically employ an ensemble framework (e.g., mixture of experts) with various features. If a feature response (e.g., the LLM’s outputs) exceeds a pre-defined latency threshold, it will receive a \textit{NULL} value during serving. Therefore, the existing SD strategy makes the rounds of verification agnostic, hindering the latency guarantees required by ranking systems in real-world applications.
\textit{\textbf{Second, ranking tasks are essentially listwise sensitive, especially unaccepted items will be the candidate tokens for future LLM decoding.}}  However, existing SD methods rely on a directed (left-to-right) multi-token prediction paradigm — i.e., not only predicting multiple independent tokens but also discarding listwise ranking knowledge about unaccepted items within each round. This leaves the fine-grained listwise ranking relationships among all items largely unexplored, e.g., how the ranking of accepted items influences the future token prediction, hindering accurate multi-token prediction.

\begin{figure}
    \centering
    \includegraphics[width=0.8\linewidth]{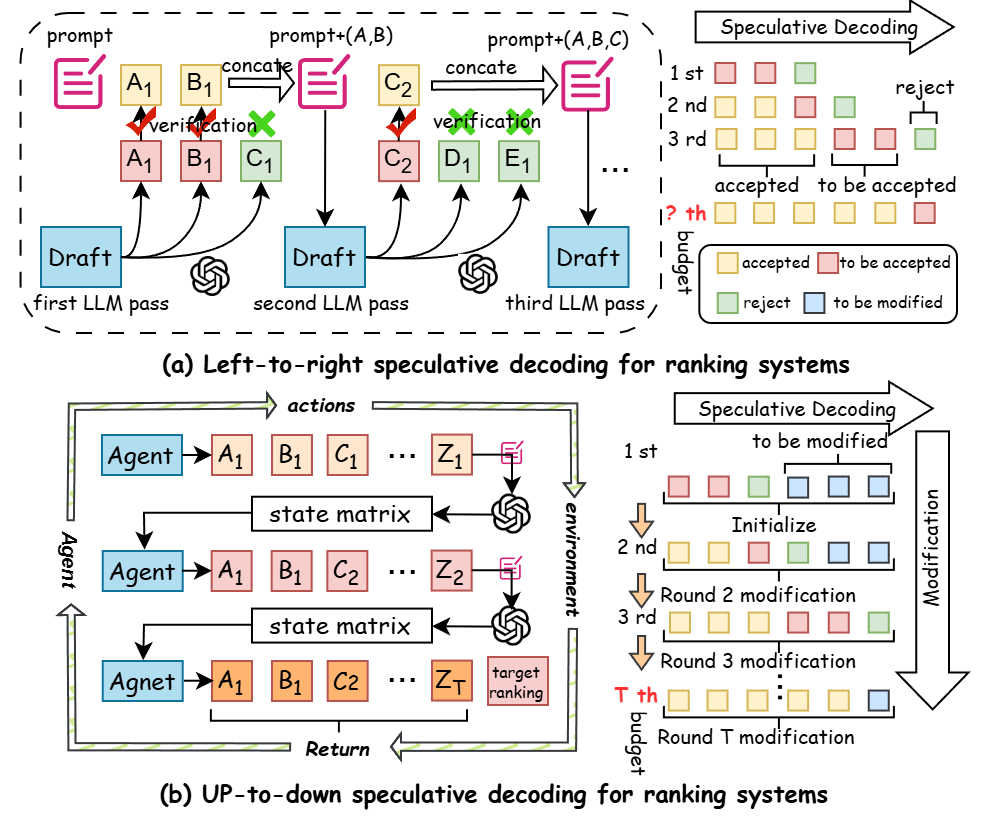}
    \caption{{\small Comparison between existing left-to-right and our up-to-down speculative decoding.}}
    \label{intro_2}
\end{figure}

To address the latency constraints in ranking systems, we propose an up-to-down decoding paradigm, as shown in \autoref{intro_2} (b), where the ranking can be iteratively modified within a constrained budget. As our goal is to find an optimal modification strategy with LLMs' verification to approximate the target ranking in the final round, simply conducting supervised learning may lead to local optima. To this end, we propose training a lightweight agent via reinforcement learning (RL) tailored for ranking-based SD. Specifically, we first pretrain the agent using supervised learning based on the target ranking of LLMs, and then fine-tune the agent by \textbf{\textit{exploring}} optimal strategies for SD. Inspired by the group relative policy optimization method \cite{shao2024deepseekmath}, we propose a ranking-tailored policy optimization in RL, parallelly sampling multiple modification actions of the agent in each round. Then, we optimize the agent based on the final return of different trajectories composed of multi-round modifications, encouraging learn optimal policy based on the high return. To better approximate the target ranking under the constrained budget, we trigger the agent fully \textbf{\textit{utilizing}} the historical knowledge of listwise rankings verified by LLMs across different rounds in RL, capturing token- and round-level dependency in item rankings for agent policy enhancement. More importantly, we theoretically elucidate the relationship between our method and established RL methodologies, demonstrating the theoretical robustness and advantages of our implementation.

In summary, we propose a Reinforcement Speculative Decoding (RSD) method for fast LLM inference in ranking systems. Specifically, we introduce a multi-round modification method that requires no fine-tuning of LLMs, and is trained via RL to iteratively modify rankings within a constrained budget. 
We theoretically elucidate the rationale and implementation of RSD, and empirically validate the effectiveness of our proposed method on both IR and RS tasks.

\section{Related Work}\label{sec:related_work}
\textbf{LLMs for ranking.}
Recently, LLMs have been widely used in ranking systems such as IR systems and RSs. Typically, these methods commonly employ a multi-stage pipeline: first selecting several candidates from the large-scale item corpus, and then ranking them by an LLM ranker. For example, several methods \cite{wang2025re2llm,wu2024survey,hou2024large} in RSs employed LLMs as recommenders to directly rank items. Recent works \cite{pradeep2023Rankzephyr,pradeep2023rankvicuna,zhu2023large} in IR tasks adopted listwise ranking to generate the ranking of candidates directly. Although the effectiveness of these methods, the extensive computational nature of LLMs makes LLM-based ranking come at the expense of increased latency \cite{meng2024ranked,parry2024top,reddy2024first,du2025active}. Among these methods, FIRST \cite{reddy2024first} is the first trial to speed up the LLM decoding, exploring lowering the number of output tokens required to be generated to one. However, it suffers from severe degradation in tail position and is hard to alleviate by the SD strategy due to distribution drift. 


\textbf{Speculative decoding (SD).}
SD aims to accelerate decoding of LLMs by drafting and verifying multi-token sequences based on the LLM encoding, ensuring output consistency to auto-regressive decoding \cite{xia2024unlocking,stern2018blockwise,xia2023speculative,leviathan2023fast}. 
According to the drafting strategies, existing methods can be categorized into \textit{independent drafting} and \textit{self-drafting} based methods. Independent-drafting-based method leverages an external low-cost model. Among these methods, directly utilizing a smaller version of the target model is a popular solution \cite{leviathan2023fast,chen2023accelerating,spectoraccelerating,chen2024cascade,chen2024sequoia}, because the availability of existing LLM series models such as OPT \cite{zhang2022opt} and LLaMA \cite{grattafiori2024llama,touvron2023llama1,touvron2023llama2}. Besides these methods, Xia et al. \cite{xia2023speculative} proposed employing a specialized Non-auto-regressive Transformer that drafts multiple tokens simultaneously per step. Self-drafting-based methods leverage the target LLM itself for efficient drafting \cite{stern2018blockwise,santilli2023accelerating,hooper2023speed,cai2024medusa,fu2024break,du2024glide}. For example, Blockwise Decoding \cite{stern2018blockwise} and Medusa \cite{cai2024medusa} incorporated FFN heads after LLM encoding, enabling the parallel token generation per round. Several methods adopt early exiting and layer skipping strategies within LLMs for drafting \cite{yangpredictive,zhang2024draft,hooper2023speed}. However, they mainly generate independent multi-tokens based on previous logits, leaving listwise ranking relations of logits and their outputs largely unexplored.
According to verifying strategies, existing methods can be categorized into independent \textit{greedy
decoding}, \textit{speculative sampling}, and \textit{token tree verification} based methods. Greedy
decoding adopts the greedy verification strategies to generate outputs that are exactly \cite{stern2018blockwise,hooper2023speed,santilli2023accelerating,yangpredictive,zhang2024draft,fu2024break} or approximately \cite{stern2018blockwise,xia2023speculative,kim2023speculative} the same as the auto-regressive decoding results. Speculative-sampling-based methods adopt different sampling methods to accelerate the target LLM’s inference without changing its original distribution, including generating an identical distribution \cite{liu2024online,chen2023accelerating,chen2024cascade,monea2023pass,yang2024multi} and an approximated distribution \cite{leviathan2023fast,zhoudistillspec}. Token-tree-verification-based methods propose to verify multiple draft sequences in a parallel way \cite{miao2024specinfer,spectoraccelerating,sun2023spectr,he2024rest,cai2024medusa,li2024eagle}. Specifically, they merge multiple candidate draft sequences into a token tree by sharing prefixes, followed by a tree-masked attention to verify the token tree. However, these methods rely on a left-to-right decoding paradigm and make the latency of decoding for all ranking items agnostic, failing to work with ranking systems that usually have strict latency requirements for response.

\section{Preliminary}
This section first introduces the background of auto-regressive decoding and speculative decoding (SD), and then formulates the problem of ranking systems, clarifying its difference from tasks addressed by existing SD methods.

\textbf{Auto-regressive decoding in LLMs.} 
LLMs take discrete token sequences as input, which is denoted as $x = (x_1, x_2, \ldots, x_s) \in \mathbb{T}^s$ of length $s$, where $\mathbb T$ denotes the token corpus. We denote $x^t_{1:m} = (x_1, x_2, \ldots, x_m)$ as a slice of $x$ of length $m$ at round $t$. The output of LLMs characterizes the probability distribution of the next token. Specifically,
The probability of the $t$-th token is determined by all previous (including generated) tokens, represented as $P_\text{llm}(x_t | x_{1:t-1})$. Then, the next input token $x_t$ is obtained by sampling from $P_\text{llm}(x_t | x_{1:t-1})$ using different methods such as greedy, top-$K$ sampling~\cite{kool2020ancestral}. We define $x^0$ as the query (or prompt) tokens, and require the LLM to generate an output sequence of length $m$ from $x^0$. Let $y_i$ denote the token generated at step $i$. In this paper, we adopt the greedy auto-regressive decoding strategy for reproducibility, which can be formulated as $y_i=\text{argmax}P_\text{llm}(y_i|x^0,y_{1:i-1})$ for $i=1,\cdots,m$. 

\textbf{Draft-then-verification paradigm in SD}. The main idea of SD is to speculate multiple potential future tokens (i.e., draft phase) and then subsequently confirm the correctness of these tokens within a single LLM encoding (i.e., verification phase).  Take greedy SD as an example. {\textit{In draft phase}}:  give the prompt $x^0$ and tokens $y_{1:t-1}$ generated so far at step $t$, we can use a draft model to generate a $n$ sequential tokens {\small$\hat{y}_{t:t+n-1}$}. {\textit{In verification phase}}, with encoding {\small$P_\text{llm}(\hat{y}_{t:t+n-1}|x^0,y_{1:t-1})$} by LLMs, we can get probabilities {\small$\{P_\text{llm}(\cdot|x^0,y_{1:t-1},\hat{y}_{t:t+i})\}_{ i= 0,\cdots ,n-1}$} by triggering the LLM \textbf{\textit{only once}} owing to its causal structure, where {\small$P_\text{llm}(\cdot | x^0,y_{1:t-1},\hat{y}_{t:t+i})$} denotes the probability of next token is determined by all previous (including prompt and generated) tokens. Finally, we can verify if {\small$\hat{y}_{t+i}=\text{argmax}_yP_\text{llm}(y|x^0,y_{1:t-1},\hat{y}_{t:t+i-1})$} for each $i$ from $i = 0$ to $i = n - 1$.  If there is a match, we accept this token and proceed; otherwise, we stop checking and drop subsequent tokens, and repeat this process until meeting token {\small$[EOS]$}. 


\textbf{Problem formulation for ranking systems.} For a query (or user) $q\in Q$, the role of LLMs is to rank $K$ candidate items $\mathcal D_q=\{d_1,\dots,d_K\}$ (e.g., documents in IRs and products in RSs) based on their reasoning (see detailed prompts $x^{0}$ in \autoref{sec:prompt}). Specifically, these candidate items can be noted by different identifiers, e.g., |A| $\sim$ |Z|, which can be used for ranking by LLMs (e.g., |B| > |A| ... > |E|). In this paper, we aim to approximate the \underline{t}arget \underline{r}anking sequence $\sigma_{tr}$ with constrained budget $T$, i.e.,
{\small $\min_{{\sigma}\gets \mathcal{M}} \text{diff}({\sigma},{\sigma_{tr}}), s.t., |\mathcal M| \le T$,} where {\small $\mathcal M=\{P_\text{llm}(b|a)\}$} denotes the necessary calls for LLM encoding to obtain predicted ranking ${\sigma}$.  Budget $T$ denotes the maximum times to trigger the LLM for sentence encoding.

\textbf{Difference to decoding tasks in existing SD methods.} SD tasks focus on the general text generation, aiming to minimize the encoding budget with approximating the generation $y$ generated by auto-regressive decoding of target LLMs, i.e., {\small$\min_{{\sigma}\gets \mathcal{M}} |\mathcal M| , s.t., \text{diff}({\sigma},{\sigma_{tr}})\le \tau$}, where $\tau$ denotes the tolerance of approximation measured by $\text{diff}(\cdot,\cdot)$ and $\tau=0$ denotes the exactly same decoding. Therefore, existing SD methods may not meet the requirements in ranking systems, leaving it an unexplored problem that is urgent to be solved.

\section{Methodology}\label{sec:method}

\textbf{Overview}. Section \ref{sec:31} introduces the up-to-down decoding paradigm under constrained budget. Section \ref{sec:32} proposes a ranking-tailored policy optimization in RL to explore optimal policy. Section \ref{sec:33} details the implementation of agents and rewards in RL by utilizing the listwise ranking knowledge. 

\subsection{Up-to-down decoding paradigm}\label{sec:31}
To approximate target ranking $\sigma_{tr}$ under a constrained budget $T$, we propose an up-to-down decoding paradigm that iteratively modifies the ranking. We denote {\small$ \sigma_{t} = [d^t_{(1)}\succ\dots\succ d^t_{(K)}]$} as the ranking modified by the agent {\small$\mathcal A$} at the $t$-th round, where {\small$d^t_{(i)}\succ d^t_{(j)}$} indicates that item {\small$d^t_{(i)}$} is ranked ahead of item {\small$d^t_{(j)}$} by the agent. Specifically, the agent {\small$\mathcal A$} follows the policy (network) {\small$\pi_\theta(\sigma_{t}| q,s_{<t})$}, which modifies the ranking to {\small$\sigma_{t}$} based on the query $q$ and the current state {\small$s_{<t}$}, where {\small$s_{<t}$} contains cached LLM encodings of the previous rankings, i.e., {\small$s_{<t}={P_\text{llm}(\sigma_i|x^{0}_q)\ |\ i=1,\cdots,t-1}$}.

To model the listwise ranking of all candidate items using the agent policy, we employ a relevance network to produce deterministic scores for all items, i.e., {\small $h_\theta(\mathcal{D}| q,s_{<t})=[h_\theta(d_1| q,s_{<t}),\cdots,h_\theta(d_K| q,s_{<t})] \in \mathbb{R}^K$}, where a higher score indicates a higher likelihood of being ranked at the top. Specifically, we adopt the Bradley–Terry probability {\small$\pi_\theta(\sigma|q,s_{<t})$}, which models a stochastic permutation to formulate the ranking $\sigma_t$ of these items, i.e.,

\vspace{-10pt}
{\small\begin{equation*}\label{eq:rum}
  \pi_\theta(\sigma_{t}| q,s_{<t}) =\Pr[d^t_{(1)}\succ\dots\succ d^t_{(K)}|q,s_{<t} \bigr]= \prod\nolimits_{d^t_i \succ d^t_j}\Pr[d^t_i\succ d^t_j|q,s_{<t}], 
\end{equation*}}\noindent where {\small $\Pr[d^t_i\succ d^t_j]=\text{Sigmoid}(h_\theta(d_i| q,s_{<t})-h_\theta(d_j| q,s_{<t}))$}. To align our agent policy with SD, we design our policy $\sigma_t\sim\pi_\theta(\cdot|q,s_{<t})$ based on the greedy verification and Bradley–Terry distribution, 

\vspace{-10pt}
{\small\begin{equation}\label{def_sigma}
\sigma_t[i] = 
\begin{cases}
\sigma_{t-1}[i] & \text{if } i \le i_t^*, \\ 
\text{argmax}P_\text{llm}(d|x^{0},\sigma_{t-1}[:i_t^*]) & \text{if }i=i_t^* +1, \\
\tilde{\sigma}_t[i]   & \text{else }  \tilde{\sigma}_t\sim \Pr[\cdot|q,s_{<t} \bigr],
\end{cases}
\end{equation}}\noindent where {\small 
$i_t^* = \max\{ i | \forall j \in [0, i],\ \sigma_{t-1}[j] = \arg\max_d P_\text{llm}(d | x^{0}, \sigma_{t-1}[:j]) \}$} denotes the longest prefix of $\sigma_{t-1}$ that is consistent with greedy decoding from model conditioned on the prompt. Note that {\small $\{P_\text{llm}(d|x^{0},\sigma_{t-1}[:i])\}_{i=0,\cdots,K}$} can be calculated by once LLM encoding {\small$P_\text{llm}(\sigma_{t-1}|x^{0})$} owing to its causal structure.  To verify the policy $\sigma_{t}\sim\pi_\theta(\cdot| q,s_{<t})$ is tailored for SD, we prove\footnote{All detail proofs in this paper can be found in \autoref{sec:proof}} the strict increasing monotonicity of $i^*_t$ for $t$.
\begin{theorem} \textbf{(Monotonicity.) }
For $i^*_T<K$, we have $i^*_t<i^*_{t+1}$ for all $t<T$.
\end{theorem}

\subsection{Ranking-tailored policy optimization in RL}\label{sec:32}

In this section, we aim to learn an optimal policy network {\small$\pi_\theta(\cdot)$} that can iteratively modify the predicted ranking toward the target ranking under a constrained budget.  

To actively \textbf{\textit{explore}} optimal multi-round ranking modification policy {\small$\pi_\theta(\cdot)$}, we parallely sample multiple modification actions of the agent in each round. Specifically, this process can be formulated as a $T$-round trajectory using RL, i.e., {\small $\sigma = (s_0,\sigma_1,r_1,\cdots,s_{<T-1},\sigma_T,r_T)$}, where the return $R_T$ is defined as the final reward $r_T$, i.e., {\small$R_T = r_T$}. To learn 
{\small$\pi_\theta(\cdot)$}, we sample $G$ independent trajectories {\small$\mathcal G_q=\{\sigma^{(1)},\dots,\sigma^{(G)}\}$} corresponding to query $q$, where each trajectory {\small $\sigma^{(i)}=[\sigma^{(i)}_1,\cdots,\sigma^{(i)}_T]$} represents the rankings modified by the policy {\small$\pi_{\theta}(\cdot)$} at different rounds. We formulate the objective of the RL process as a Ranking-tailored Policy Optimization (RPO) goal:

\vspace{-10pt}
{\small\[
    \mathcal{L}_{RPO} = \underbrace{\mathbb{E}_{[q\sim Q,\mathcal{G}_q\sim\pi_{\theta}(\cdot|q)]}\frac{1}{G}\sum_{i=1}^{G}\frac{1}{T}\sum_{t=1}^T\hat{A}_i\log\pi_\theta(\sigma_t^{(i)}|q,s^{(i)}_{<t})}_{\mathcal{L}_{RPO-1}} + \beta\cdot \text{KL}\left(\frac{h_\theta(\mathcal{D}_q| q,s^{(i)}_{<t})}{h_{ref}(\mathcal{D}_q|q,s^{(i)}_{<t})}\right),
\]}

\vspace{-10pt}
\noindent where $\pi_\theta$ and $\pi_{ref}$ are the policy network and reference model, and {\small$h_\theta$} and {\small$h_{ref}$} denote their relevance networks.
$q$, {\small$\sigma_t^{(i)}$} denote the query sampled from the question set and ranking modified in $t$-th round based on the policy {\small$\pi_\theta(\sigma_t^{(i)}|q,s^{(i)}_{<t})$}, respectively. {\small$s^{(i)}_{<t}$} denote the state that contains information before $t$-round. {\small$\hat{A}_i$} denotes the advantage calculated based on the relative rewards of the outputs, whose implementation is detailed in Section \ref{sec:advantage}.  Theoretically, the proposed RPO objective can be seen as the ranking-tailored GRPO objective \cite{shao2024deepseekmath} when {\small$\pi_{\theta}/\pi_{\theta_{old}}\to 1$} and {\small$\mathcal{G}\sim\pi_{\theta_{old}}$} as follows:

\begin{theorem}
 If we have $\pi_{\theta}/\pi_{\theta_{old}}\to 1$ (i.e., $|\pi_{\theta}/\pi_{\theta_{old}}- 1|<\epsilon$) and $\mathcal{G}\sim\pi_{\theta_{old}}$, we have our RPO objective is equivalent to GRPO objective (Equation (3) in \cite{shao2024deepseekmath}) in the first term w.r.t. the model parameter $\theta$, that is,

 \vspace{-15pt}
 {\small 
\begin{equation*}
    \mathcal{L}_{RPO-1} \approx \mathbb{E}_{Q,\pi_{\theta_{old}}}\frac{1}{G}\sum_{i=1}^{G}\frac{1}{T}\sum_{t=1}^T\{\min\left[\frac{\pi_\theta(\sigma_t^{(i)}|q,s^{(i)}_{<t})}{\pi_{\theta_{old}}(\sigma_t^{(i)}|q,s^{(i)}_{<t})}\hat{A}_i,\text{clip}\left(\frac{\pi_\theta(\sigma_t^{(i)}|q,s^{(i)}_{<t})}{\pi_{\theta_{old}}(\sigma_t^{(i)}|q,s^{(i)}_{<t})},1-\epsilon,1+\epsilon \right)\hat{A}_i\right] + C
\end{equation*}}\noindent where $C$ is a $\theta$-independent constant.
\end{theorem}


In summary, the proposed policy $\pi_\theta(\cdot)$ for ranking can not only operate in a SD manner, as demonstrated by Theorem 1,  but also can be optimized in the way of the famous GRPO strategy. Therefore, the proposed RL framework satisfies the latency requirements of ranking systems by iteratively modifying the ranking under a constrained budget. This makes it well-suited for both agent training and serving within the up-to-down decoding paradigm.

\subsection{Reinforcement speculative decoding method in ranking systems}\label{sec:33}
This section introduces the implementation of our RSD method, illustrates the theoretical rationale underlying its design, and details the training strategy and time complexity analysis of RSD.

\subsubsection{Relevance network $h_\theta(\cdot)$ for deterministic scores}
This section centers our relevance network design on the following question: \textit{How can we alleviate the drawbacks of existing SD methods — not only do they predict for multiple independent tokens, but also they tend to discard listwise ranking knowledge of unaccepted items within each round?}

Intuitively, the LLM encoding of the predicted ranking generated by the agent contributes to not only verifying the accepted items but also revealing listwise relationships among different candidate items. Specifically, the LLM encoding {\small$S_{t} = P_\text{llm}(\sigma_t|x^{0}_{q}) \in \mathbb{R}^{K \times K}$} represents the probability of selecting the next item based on given ranked items. On the one hand, {\small$S_{t}[m,:]$} denotes the probability of selecting the next item with the prefix ranking {\small$\sigma_t[:m\!-\!1]$}. This indicates how rankings of candidate items mutually affect each other in a probabilistic way, helping to capture token-level dependency instead of the independent multi-token prediction assumption in existing SD methods. On the other hand, {\small$\sigma_0\!\to\! S_0\to\!\cdots\!\to\! \sigma_T$} contains rich knowledge about listwise ranking of all items verified by LLMs across different rounds in RL, helping to capture dependency in round-level item rankings for agent policy enhancement. To this end, we propose to model the sequential patterns (i.e., dependency) of listwise rankings at both the token level (inter-$S_t$) and the round level (across {\small$S_1, \cdots, S_t$)},

\vspace{-15pt}
{\small\begin{align*}
    h_\theta(\mathcal{D}|q,s_{<t})=Softmax(\text{Mean}(Z_{t-1}))\in\mathbb R^K, \quad Z_0,\cdots,Z_{t-1} = \text{Transformer}([S_{0},\cdots,S_{t-1}];\theta).
\end{align*}}\noindent where {\small$Z_j\in \mathbb R^{K\times K}$ and $Mean(\cdot)$} denotes average pooling for all items' encoding by the Transformer model. Therefore, the relevance network produces deterministic scores for all items {\small $h_\theta(\mathcal{D} \mid q, s_{<t})$}, indicating that items with higher scores are more likely to be ranked at the top.

\subsubsection{Selection for advantage $\hat{A}_i$ and reference model $\pi_{\theta_{ref}}(\cdot)$}\label{sec:advantage}


To measure the gap between predicted ranking $\sigma_T$ and the LLM's target ranking $\sigma_{tr}$, we formulate the final return by the Spearman distance (detailed in \autoref{ranking_measure}).
With the return $R_i$ for each the trajectory {\small$\sigma^{(i)}$}, we have two strategies to implement the advantage in group {\small$\mathcal G$}, namely group average advantage {\small$\hat{A}_i^\text{group}$} and reference advantage {\small$\hat{A}_i^\text{ref}$}, i.e.,

\vspace{-10pt}
{\small\begin{equation*}
    \hat{A}_i^\text{group} = R_i - \hat{\mu}_{-i}; \quad \hat{\mu}_{-i}=\frac{1}{G-1}\sum\nolimits_{j\neq i}  R_j  \quad and \quad \hat{A}_i^\text{ref} = R_i - R_\text{ref}; 
\end{equation*}}

\begin{theorem}
By  assigning $\pi_{\theta_{ref}} \gets \pi_{\theta}$ in the training process, we have both updating the gradient in $\mathcal{L}_{RPO-1}$ w.r.t. group average advantage $\hat{A}_i^\text{group}$ and reference advantage $\hat{A}_i^\text{ref}$ is unbiased, i.e.,
{\small\begin{equation*}
 \mathbb E[(R_i-B)\cdot \nabla_{\theta_{d}}\log\pi_{\theta}(\sigma_t|q,s_{<t})]=\mathbb E[R\cdot\nabla_{\theta_{d}}\log\pi_{\theta}(\sigma_t|q,s_{<t})]
\end{equation*}}\noindent  where $B=\hat{\mu}_{-i}$ or $B = R_\text{ref}$. Suppose the return of {\small$R_i= \mu + \varepsilon_i$ and $R_{\text{ref}} = \mu + \delta$} can be decomposed into query-level reward (related to the difficulty of the question), i.e., {\small$\mu\sim \mathcal N\!(0,\sigma_b^{2}\bigr)$}, and  temperature noise (related to the random sampling), i.e., {\small$\delta\sim \mathcal N\!(0,\sigma_\delta^{2}\bigr)$} and {\small$\varepsilon_i \stackrel{\text{i.i.d.}}{\sim} \mathcal N\!(0,\sigma_w^{2}\bigr)$} with {\small$\varepsilon_i\perp\delta$}, we have their variance 

\vspace{-10pt}
{\small\begin{equation*}
    Var[(R-R_\text{ref})\cdot \nabla_{\theta_{d}}\log\pi_{\theta}(\sigma_t|q,s_{<t})] < Var[(R-\hat{\mu}_{-i})\cdot \nabla_{\theta_{d}}\log\pi_{\theta}(\sigma_t|q,s_{<t})]
\end{equation*}}
 if and only if  {\small$\sigma_\delta^{2} < {\sigma_w^{2} }/{(G-1)}$}.
\end{theorem}


This theorem proves the \emph{unbiasedness} and characterizes the \emph{variance} of our gradient estimator when using either the \textbf{group-average} baseline $\hat{\mu}_{-i}$ or the \textbf{reference} baseline $R_{\text{ref}}$.
Specifically, both choices yield unbiased gradient estimation. Moreover, Theorem~3 shows that the reference baseline achieves strictly lower gradient variance whenever $\sigma_\delta^{2} < {\sigma_w^{2}} / (G-1)$,
where $\sigma_w^{2}$ denotes the temperature noise of trajectory rewards,
and $\sigma_\delta^{2}$ represents the variance of the reference reward.
Therefore, we adopt the reference model with \textbf{greedy decoding}, i.e., $\sigma_t^{\text{ref}} \gets \max_{\sigma_t} \pi_\theta(\sigma_t \mid q, s_{<t})$ for $t = 1, \cdots, T$,
which yields $\sigma_\delta^{2} \approx 0$, thereby automatically satisfying the condition for any group size $G > 1$.
As a result, the greedy reference model with advantage $\hat{A}_i^\text{ref}$ enables more stable updates during RPO training.

\subsubsection{Training strategy and time complexity analysis.}

We implement the training of the RSD method into two stages: 1) initialize the policy network by supervised learning; 2) fine-tune the policy network through RPO in RL. We detail the pseudo-algorithm of the RSD method in Algorithm \autoref{alg:rsd}.
\begin{algorithm}[H]
\caption{Pseudo algorithm of RSD method}
\label{alg:rsd}
\begin{algorithmic}[1]
{\small
\Require Queries $Q$, Candidate items $\mathcal{D}$, Target LLM $P_\text{llm}$, Budget $T$, Group size $G$

\State \textbf{// Stage I: Supervised initialization}
\For{$q \in Q$}
   \State Initialize ranking $\sigma_{init}$ and compute LLM prob $S_{init}$.
   \State Update policy network based on supervised learning.
\EndFor

\State \textbf{// Stage II: Ranking-tailored policy optimization}
\For{$q \in Q$}
    \State Compute reference ranking $\sigma_t^{\text{ref}}$ and reward $R_{\text{ref}}$
    \For{$i = 1$ to $G$}
        \For{$t = 1$ to $T$}
            \State Sample ranking $\sigma_t^{(i)} \sim \pi_\theta(\cdot | q, s^{(i)}_{<t})$
            \State Compute logits $S_t^{(i)} = P_\text{llm}(\sigma_t^{(i)} | x^{0}_q)$
            \State Update $s^{(i)}_{<t+1} \gets s^{(i)}_{<t} \cup S_t^{(i)}$
        \EndFor
        \State Compute return $R^{(i)}$ and advantage $\hat{A}^\text{ref}_i$
    \EndFor
    \State Update policy network according to $\mathcal L_{RPO}$.
\EndFor
}
\end{algorithmic}
\end{algorithm}

\textbf{Stage~I: Supervised Initialization}. To provide a strong and warm start for the subsequent RPO fine-tuning, we first train the policy network {\small$\pi_{\theta}(\cdot)$} with plain supervised learning. Specifically, we initialize the ranking {\small$\sigma_{init}$} based on the probability {\small$P_\text{llm}(d|x^{0}_q)$} on candidate items $\mathcal D_q$, and compute the LLM probability {\small$S_{init} = P_\text{llm}(\sigma_{init}| x^{0}_q)$}. Then, we train the policy network based on supervised learning, i.e., {\small$\max_\theta\ln\pi_\theta(\sigma_{tr}|q,S_{init})$}, where {\small$\sigma_{tr}$} denotes the target ranking. 

\textbf{Stage~II: Ranking-tailored Policy Optimization.} With the 
warm-start, we further fine-tune the policy network by RL so that it can \emph{actively} explore optimal policy and utilize listwise ranking knowledge under a fixed call budget, including trajectory sampling, return and advantage calculation, ranking-tailored policy optimization.

\textbf{Inference complexity analysis}. Suppose we have $M$ tokens in the prompt and $K$ candidate items for ranking, the decoding complexity of our  RSD approximates {\small$\mathcal{O}(T\cdot(M+K)^2\cdot o)$}, where $o$ denotes the overall of dimension of the target LLM. It is lower than auto-regressive decoding with complexity {\small$\mathcal{O}([M^2+\sum_{k=1}^{K}(K+k)^2]\cdot o)$} as {\small$T \ll K$}, and existing SD with complexity {\small$\mathcal{O}(T_{sp}\cdot(M+K)^2\cdot o)$} when {\small$T < T_{sp}$}, where $k$ denotes indicator for $k$-th token generation and $T_{sp}$ denotes the required times of LLM encoding for SD methods. With the KV cache mechanism, the incremental complexity of RSD for decoding is {\small$\mathcal{O}(T\cdot(M+K)\cdot o)$}, which is also lower than auto-regressive decoding with complexity {\small$\mathcal{O}(\sum_{k=1}^{K}(K+k)\cdot o)$} and existing SD with complexity {\small$\mathcal{O}(T_{sp}\cdot(M+K)\cdot o)$} when {\small$T < T_{sp}$}. In summary, our RSD not only shows lower complexity compared to auto-regressive decoding, but also shows controllable complexity compared to SD methods.

\section{Experiments}
In this section, we aim to validate the effectiveness of the proposed method RSD. Specifically, we conduct extensive experiments in both IR and RS tasks to study the following research questions:  \textbf{RQ1}: Whether the proposed RSD outperforms state-of-the-art STD methods and SD methods? \textbf{RQ2}: Whether the proposed RSD benefits from the proposed module, such as RPO, listwise ranking modeling, proposed advantage function? \textbf{RQ3}: Whether the proposed RSD method trained on specific data and backbone can be used for others (e.g., different LLM backbones, datasets)?  
\textbf{RQ4}: How do hyperparameters influence the performance of the proposed RSD?

\renewcommand{\arraystretch}{1}
\begin{table*}
  \centering
  \fontsize{6}{9}\selectfont
  \caption{{\small Performance of different methods on \textbf{IR tasks and RS tasks} with different LLM backbones.  * indicates statistically significant improvement of RSD to baselines on t-test (p < 0.05).}}
  \label{tab:RSComparison}
\begin{tabular}{|c|c|c|c||c|c|c|c|c|c|c|c|c|}
\midrule\multicolumn{2}{|c|}{{\fontsize{6pt}{6pt}\selectfont Task/LLM}}                                      & Dataset                       & {\fontsize{5pt}{6pt}\selectfont Method}   & SDM    & {\fontsize{6pt}{6pt}\selectfont Blockwise} & Medusa & PDM    & {\fontsize{6pt}{6pt}\selectfont Rankzephyr} & FIRST  & {\fontsize{5pt}{6pt}\selectfont FIRST+SD} & RSD            & Imprv.  \\\midrule\midrule
\multirow{8}{*}{\rotatebox{90}{IR task}} & \multirow{8}{*}{\rotatebox{90}{Llama-3.2-3B}} & \multirow{4}{*}{\begin{tabular}[c]{@{}c@{}}{\fontsize{5pt}{6pt}\selectfont MS}\\ {\fontsize{5pt}{6pt}\selectfont MARCO}\end{tabular}}   
                & KT$\uparrow$       & 0.4962 &0.5801    & 0.5462 & 0.5436 & 0.2693     & 0.3681 & 0.4047     & \textbf{0.7169*} & 23.58\% \\
&            &   & SR$\uparrow$        & 0.6196 & 0.7027    & 0.6746 & 0.6665 & 0.3654     & 0.4836 & 0.5154     & \textbf{0.8371*} &19.13\%  \\
&            &   & FD$\downarrow$       & 71.084 & 59.616  & 63.072 & 99.956 & 101.07    & 87.800 & 80.600     & \textbf{40.964*} & 31.29\% \\
&            &   & KD$\downarrow$      & 47.858 & 39.890   & 43.114 & 68.460 & 69.414     & 60.028 & 56.554     & \textbf{26.896*} & 32.57\% \\\cline{3-13}
&            & \multirow{4}{*}{Quora}        & KT$\uparrow$      & 0.4267 & 0.5829    & 0.5461 & 0.5393 & 0.2399     & 0.3684 & 0.4737     & \textbf{0.6454*} & 10.7\% \\
&            &   & SR$\uparrow$       & 0.5514 & 0.7079    & 0.6793 & 0.6625 & 0.3402     & 0.4973 & 0.6015     & \textbf{0.7824*} & 10.5\% \\
&            &   & FD$\downarrow$     & 123.97 & 91.736    & 98.212 & 100.95 & 161.72     & 135.92 & 111.67     & \textbf{78.808*} & 14.1\% \\
&            &   & KD$\downarrow$      & 85.999 & 62.566    & 68.088 & 69.100 & 114.02     & 94.740 & 78.940     & \textbf{53.186*} & 15.0\% \\\midrule\midrule
\multirow{8}{*}{\rotatebox{90}{RS task}} & \multirow{8}{*}{\rotatebox{90}{Qwen2.5-7B}}   & \multirow{4}{*}{{\fontsize{5.5pt}{6pt}\selectfont ML-1M}}        & KT$\uparrow$       & 0.4841 & 0.5745    & 0.5481 & 0.5425 & 0.1947     & 0.0915 & 0.2561     & \textbf{0.6426*} & 11.9\% \\
&            &   & SR$\uparrow$       & 0.6157 & 0.7086    & 0.6852 & 0.6647 & 0.2816     & 0.1401 & 0.3356     & \textbf{0.7659*} & 8.09\%  \\
&            &   & FD$\downarrow$     & 111.99 & 94.512    & 98.970 & 102.25 & 166.77     & 186.18 & 154.20     & \textbf{78.568*} & 16.9\% \\
&            &   & KD$\downarrow$     & 77.390 & 63.828    & 67.782 & 68.632 & 120.80     & 136.27 & 111.59     & \textbf{53.614*} & 16.0\% \\\cline{3-13}
&            & \multirow{4}{*}{\begin{tabular}[c]{@{}c@{}}{\fontsize{5.5pt}{6pt}\selectfont Amazon}\\ {\fontsize{5.5pt}{6pt}\selectfont -Games}\end{tabular}} & KT$\uparrow$        & 0.4726 & 0.5795    & 0.4924 & 0.5106 & 0.0888     & 0.0850 & 0.0342     & \textbf{0.6404*} & 10.5\% \\ 
&            &   & SR$\uparrow$     & 0.6104 & 0.7143    & 0.6283 & 0.6343 & 0.1256     & 0.1213 & 0.0459     & \textbf{0.7716*} & 8.02\%  \\
&            &   & FD$\downarrow$     & 115.09 & 92.732    & 109.98 & 108.11 & 193.84     & 194.32 & 194.69     & \textbf{80.376*} & 13.3\% \\
&            &   & KD$\downarrow$      & 79.113 & 63.078    & 76.144 & 73.414 & 136.68     & 137.25 & 144.87     & \textbf{53.946*} & 14.5\%\\\midrule\midrule
\end{tabular}
\vspace{-15pt}
\end{table*}

\subsection{Experimental Setup}
\textbf{Tasks and Datasets.} In this paper, we validate the effectiveness of the proposed method RSD on two representative tasks in ranking systems \cite{liu2022generalized}, namely the IR task and the RS task. The goal of the IR task is to return $K$ passages that are most relevant to a query $q$. For IR tasks, we adopt the MS MARCO dataset\footnote{{\small\url{https://huggingface.co/datasets/rryisthebest/rank_zephyr_training_data_alpha}}} provided by \cite{reddy2024first} and retain examples with 20 candidate passages that need to be reranked, resulting in a total of $9,978$ examples. In addition, we adopt the Quora dataset\footnote{{\small\url{https://public.ukp.informatik.tu-darmstadt.de/thakur/BEIR/datasets/quora.zip}}} collected from Quora.com and randomly select $5,000$ queries, each associated with $25$ candidate passages that are most relevant to the query in the corpus based on the BM25 model. The goal of the RS task is to provide users with items they may prefer. For RS tasks, we adopt the ML-1M dataset\footnote{{\small\url{https://grouplens.org/datasets/movielens/}}} collected by the MovieLens website with $6,040$ users, and the Amazon-Game dataset\footnote{{\small\url{https://jmcauley.ucsd.edu/data/amazon/}}} collected from Amazon.com with $9,512$ users. For each user in these two datasets, we keep $95\%$ of their interactions as historical behaviors and use the remaining interactions to construct candidate items. Specifically, we construct $25$ candidate items for each user by randomly selecting from the user's remaining interactions, potentially preferred items based on user-KNN methods, and popular items.

\textbf{Backbone models and baselines.}
In this paper, we select two representative LLM backbones with different sizes and series, namely Llama-3.2-3B(-Instruct) \cite{grattafiori2024llama} and Qwen2.5-7B(-Instruct) \cite{yang2024qwen2}. For baseline methods, we take the following state-of-the-art methods with STD strategy, i.e., target LLM with STD, Rankzephyr, FIRST, and methods with SD strategy, i.e., Speculative Decoding Method (SDM) \cite{leviathan2023fast}, Blockwise \cite{stern2018blockwise}, Medusa \cite{cai2024medusa}, and Parallel Decoding Method (PDM) \cite{santilli2023accelerating}, and hybird method with combining FIRST and SD strategy, i.e., FIRST + Speculative Decoding. As there is significant difference between our task and those of these methods, necessary modifications need to be adopted for fair comparison, which is detailed in \autoref{sec_baseline}.

\textbf{Evaluation Protocol, Metrics, and Implementation Details.}
For each dataset, we reserve the last 500 queries for the validation set and the preceding 500 queries for the test set, while the remaining queries are used for training. During the evaluation phase, we adopt four metrics \cite{murphy2012machine} to measure the similarity between the predicted ranking and the target ranking, namely Kendall’s Tau (KT), Spearman’s Rho (SR), Footrule Distance (FD), and Kemeny Distance (KD). Experimental results are reported as the average over five runs. For implementation details, we fix the budget $T=5$ for our method and the SD methods to ensure fair comparison. We set the learning rate to $5\mathrm{e}{-5}$ using the Adam optimizer and a batch size of $16$ for all methods. For our method, we adopt a single-layer Transformer with $5$ heads and a hidden dimension of $25$. We set the coefficient of the KL loss from the reference model to $0.1$. For baseline methods, we use the parameter settings provided by the authors when available; otherwise, we tune them for optimal performance under the constrained budget and fair computational complexity. All experiments were conducted on a unified setup featuring an \textit{AMD EPYC  7543} 32-core processor and four \textit{NVIDIA A6000} 48GB GPUs. The source code is available when the paper is accepted. 

\subsection{Experiment results and analysis}
\textbf{Model Comparison (RQ1).}
We investigate whether the proposed RSD outperforms state-of-the-art methods on IR tasks and RS tasks.
\autoref{tab:RSComparison} shows the performance of different methods, including STD based methods and SD-based models on IR and RS tasks with different LLM backbones, where remaining experimental results are shown in \autoref{appendix_table1} and \autoref{appendix_table2}. From the experimental results, we obtain the following conclusions: Firstly, the proposed method RSD significantly outperforms all baseline methods in all cases, which shows the effectiveness of the proposed method. RSD demonstrates notable enhancements, with average improvements of 19.61\% and 12.39\% when compared to the top-performing baseline models on the IR task and RS task, respectively. Secondly, the SD-based methods (e.g., SDM, Blockwise, Medusa, and PDM) outperform the STD based methods (e.g., Rankzephyr and FIRST), showing increasing budgets for draft-then-verification contributing to aligning predicted target ranking. Although we can combine the FIRST and SD method (i.e., FIRST+SD), limited improvement is achieved due to the distribution drift of FIRST in terms of the target LLM model. Thirdly, the methods with trainable heads (e.g., Blockwise and Medusa) achieve better performance than the methods without trainable heads (e.g., SDM and PDM) in most cases, indicating the necessity of bridging between LLM encoding (e.g., logits and last layer embedding) and draft model for in ranking tasks. Finally, fine-tuning LLMs (e.g., Rankzephyr and FIRST) even leads to worse performance in Qwen2.5-7B and RS task, showing that adopting SFT and learn-to-rank loss \cite{burges2005learning} is not a generalized and tailored solution for LLM inference approximation.

\begin{wraptable}{r}{0.45\textwidth}
  \centering
  \fontsize{7}{9}\selectfont
      \vspace{-15pt}
  \caption{{\small Ablation studies, where the complete experimental results are show in \autoref{appendix_ablation}}}
  \label{tab:ablation}
\begin{tabular}{|c|c|c|c|c|c|}\midrule
\multirow{2}{*}{LLM}                      & Dataset & \multicolumn{2}{c|}{MS MARCO}           & \multicolumn{2}{c|}{Amazon-Games} \\\cline{2-6}
  & Metrics  & KT$\uparrow$              & FD$\downarrow$              & KT$\uparrow$               & FD$\downarrow$                          \\\midrule
\multirow{6}{*}{\rotatebox{90}{Llama-3.2-3B}} & STD     & 0.3042          & 96.26                & 0.0889          & 190.6          \\
& GSD     & 0.6360          & 49.52           & 0.5839          & 88.20          \\
& w/o RPO & 0.6797          & 46.05                & 0.6766          & 70.16          \\
& w/o  LRK & 0.6065          & 55.10                & 0.6479          & 76.20          \\
& w/o RA  & 0.6907          & 44.80                 & 0.6880          & 67.99          \\
   & RSD   & \textbf{0.7169}          & \textbf{40.96} & \textbf{0.6901} & \textbf{68.38} \\\midrule
\multirow{6}{*}{\rotatebox{90}{Qwen2.5-7B}} & STD     & 0.3529          & 88.04            & 0.2328          & 160.4          \\
& GSD     & 0.6337          & 50.10                & 0.5787          & 88.81          \\
& w/o RPO & 0.6823          & 45.42               & 0.6571          & 76.36          \\
& w/o  LRK & 0.5343          & 64.00               & 0.5551          & 97.53          \\
& w/o RA  & 0.6789          & 45.84               & 0.6515          & 77.55          \\
& RSD   & \textbf{0.6996} & \textbf{43.41} & \textbf{0.6832} & \textbf{71.04}\\\midrule
\end{tabular}
    \vspace{-15pt}
\end{wraptable}

\textbf{Ablation studies (RQ2).} To validate the contributions of each critical component in our RSD method, we conducted comprehensive ablation studies. Specifically, we examined the following variants: \textbf{STD} (Single Token Decoding): A baseline variant that relies solely on single-token decoding of prompt encodings without SD. \textbf{GSD} (Greedy Speculative Decoding): A variant using speculative decoding guided purely by greedy token selection based on the logits from the encoding of the target LLM. \textbf{w/o RPO}: Removing the RPO stage, where the agent is trained only with supervised learning. \textbf{w/o LRK} (Listwise Ranking Knowledge): Removing modeling of listwise ranking knowledge among candidate items, i.e., only adopting the next-token logits of $\sigma_t[i_t^*]$ as features, and training an MLP head as the relevance network. \textbf{w/o RA} (Reference Advantage): Replacing the reference advantage {\small $\hat{A}^\text{ref}$} with the group-average advantage {\small $\hat{A}^\text{group}$} during RL training. \textbf{RSD}: The proposed method.

\begin{wrapfigure}{l}{0.35\linewidth}
    \vspace{-10pt}  
    \centering
    \includegraphics[width=1.0\linewidth]{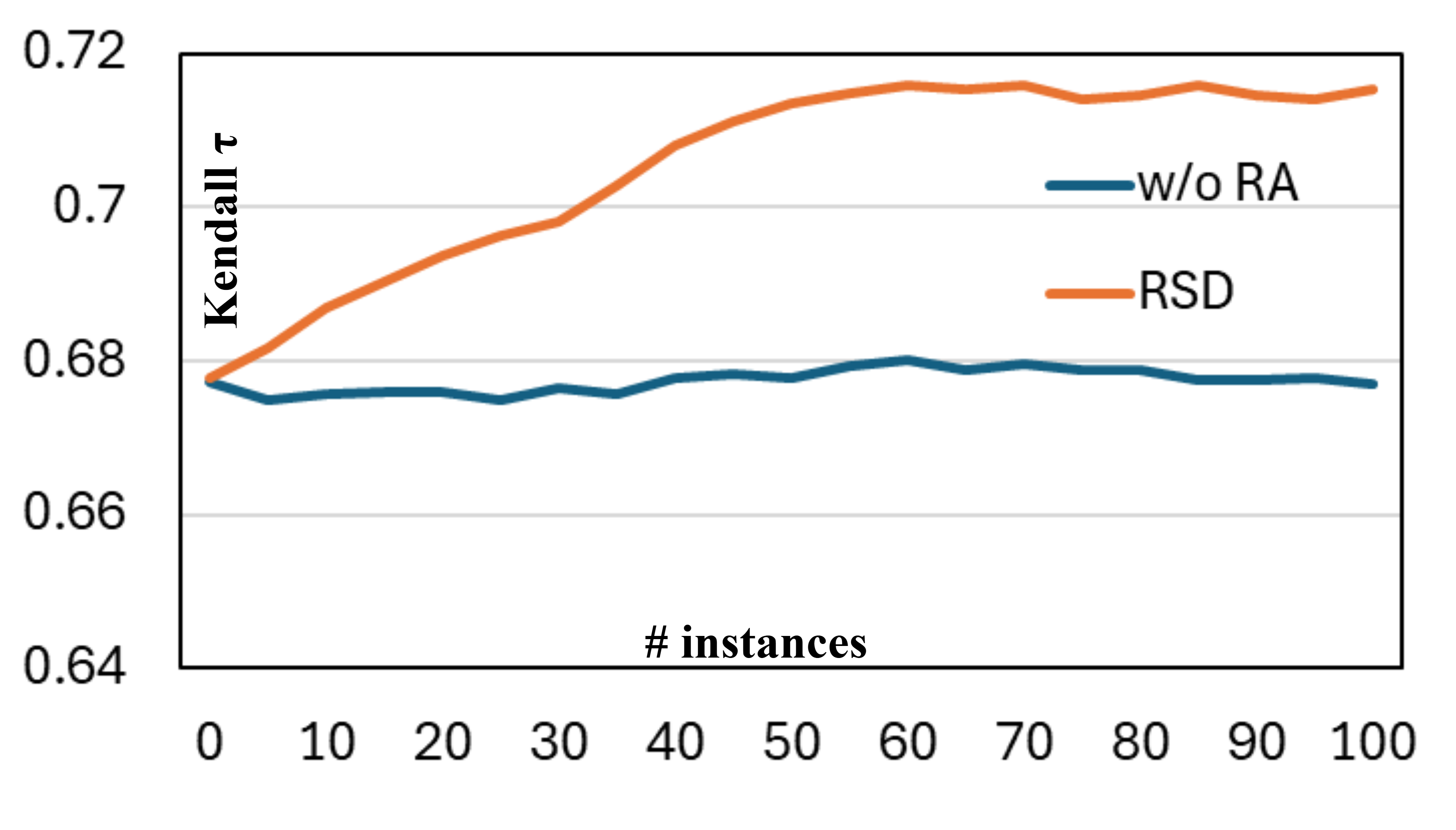}
    \vspace{-25pt}
    \caption{{\small Performance of RSD and w/o RA in training process.}}
    \label{pic_training}
    \vspace{-10pt}  
\end{wrapfigure}
\autoref{tab:ablation} reports the detailed results of these ablation experiments on both IR and RS tasks using Llama-3.2-3B and Qwen2.5-7B models. From these results, we draw the following observations:
First, the variant STD shows the worst performance among all methods, demonstrating the necessity of increasing the budget for better LLM decoding approximation. Although employing the self-drafting of the target LLM, as demonstrated by GSD, it is not a tailored approach for LLM decoding in ranking systems with increasing budgets.
Second, the variant w/o RPO shows the best performance among all methods except for RSD, indicating the effectiveness of modeling fine-grained listwise ranking for SD. Specifically, the variant w/o LRK performs poorly despite being equipped with the up-to-down decoding paradigm, illustrating the effectiveness of our Transformer-based agent in learning sequential (dependency) patterns of token- and round-level listwise rankings.
Third, the performance gap between w/o RPO and full RSD demonstrates the critical role of the RL fine-tuning stage. Without RPO, the model achieves suboptimal exploration of ranking modification strategies, resulting in less accurate approximations.
Finally, the variant "w/o RA" with group-average advantage shows performance degradation compared to the method RSD with reference advantage, underscoring that incorporating a reference model in advantage estimation effectively stabilizes and improves the RL training process. Specifically, \autoref{pic_training} shows the performance (KT) of RSD and the variant w/o RA varying with training batches on the MS MARCO dataset with Qwen2.5-7B, where each batch contains $G$ trajectories sampled based on one query. On the one hand, RSD with the reference advantage indicates stable improvement compared to the variant w/o RA, which shows marginal and unstable improvement. On the other hand, only a few (e.g., $<70$) instances are required in the RL phase, showing the good generalization ability of our RSD.

\begin{wrapfigure}{r}{0.5\linewidth}
    \centering
    \vspace{-10pt}
    \includegraphics[width=1\linewidth]{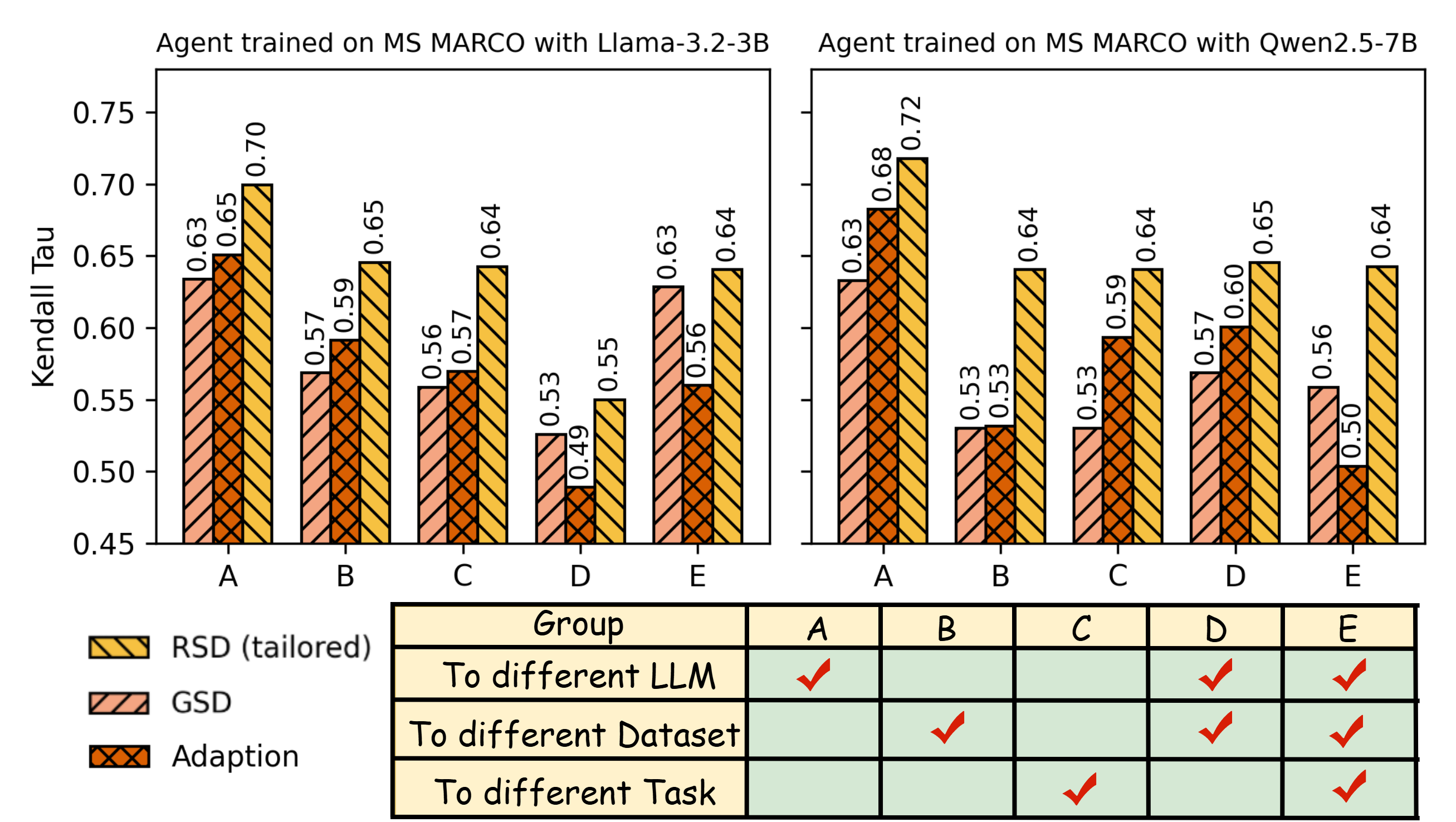}
    \vspace{-15pt}
    \caption{\small Investigate on adaptation ability of RSD.}
    \label{fig:adaption}
    \vspace{-5pt}  
\end{wrapfigure}
\textbf{Investigate on adaptation ability of the agent to different scenarios (RQ3).}
In this part, we aim to investigate the adaptation ability of RSD. Specifically, we train an agent based on specific data and a backbone, evaluating its adaptation on other LLM backbones, datasets, and tasks. \autoref{fig:adaption} shows the performance of the adapted agent (Adaption), the tailored agent (RSD), and GSD, where Adaption and GSD are not trained based on specific data and backbone, unlike RSD.
Firstly, Adaptation outperforms GSD in most cases of scenarios A, B, and C, indicating that there are some common patterns that can be effectively captured across backbones, datasets, and tasks in RL.
Secondly, Adaptation degrades compared to GSD in most cases of scenarios D and E, indicating the difficulty of adapting an agent to both different backbones and different tasks/datasets.
Thirdly, RSD shows the best performance in all cases, indicating that it is better to tailor the agent for real-world applications.

\begin{figure}[H]
    \centering
    \includegraphics[width=1.0\linewidth]{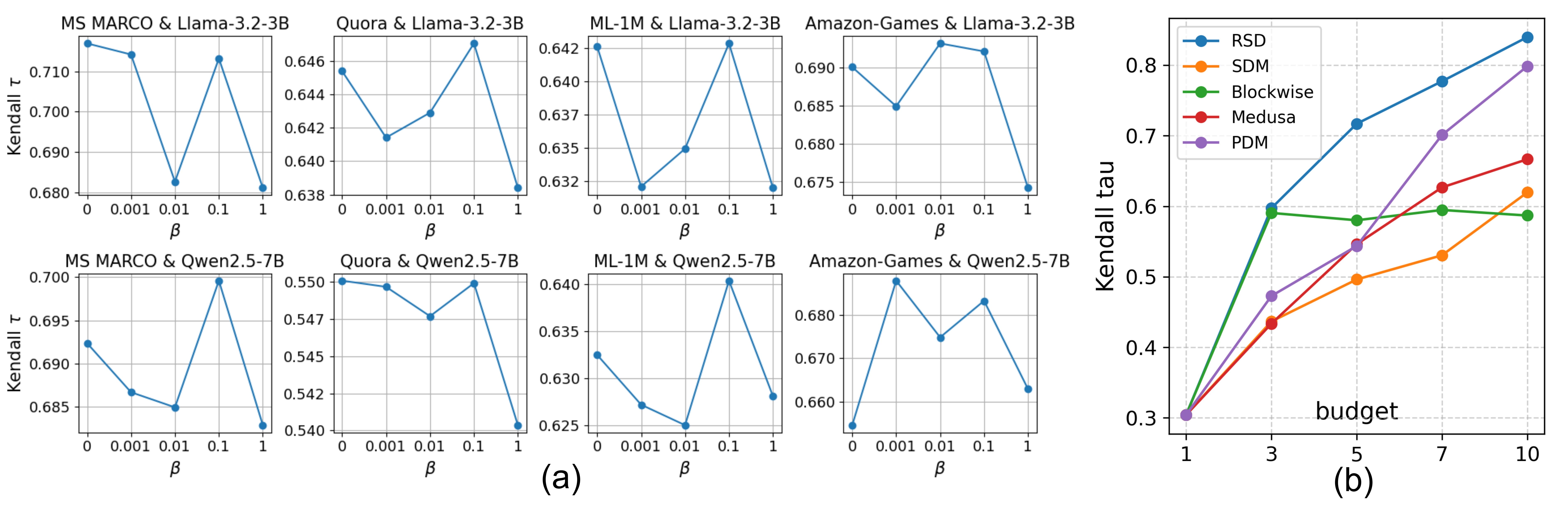}
    \vspace{-25pt}
    \caption{{\small Hyper-parameters analysis on (a) KL coefficient $\beta$ and (b) budget.}}
    \label{fig:hyper}
\end{figure}
\textbf{Hyper-parameter analysis (RQ4).} In this part, we investigate how hyperparameters influence the performance of the proposed RSD, including the KL coefficient $\beta$ and budget $T$. As depicted in \autoref{fig:hyper} (a), we observe that the highest performance is achieved when the KL coefficient $\beta=0.1$ in most cases. Therefore, we suggest implementing $\beta=0.1$ in real-world scenarios. We also investigate how the budget influences our method and others on the MS MARCO dataset with Llama-3.2-3B, as shown in \autoref{fig:hyper} (b). On the one hand, our RSD outperforms other baseline methods in all cases, especially showing considerable improvement with a limited budget. On the other hand, most methods improve with more budget in most cases. In real-world scenarios, the budget can be flexibly selected based on the needs of ranking systems.

\section{Conclusion and limitation}
\textbf{Conclusion}. In this paper, we introduce a Reinforcement Speculative Decoding method for fast LLM inference in ranking systems, featuring a multi-round modification method that operates without fine-tuning the target LLMs.
RSD adopts an up-to-down decoding paradigm that iteratively refines rankings within a constrained 
budget. Beyond empirical advancements, we provide a theoretical analysis of the principles and implementation of RSD. Experiments on IR and RS tasks demonstrate that our method significantly reduces latency while maintaining high-quality performance. 

\textbf{Limitation.} There are two main limitations in our RSD. Firstly, although RSD can be adapted to new LLMs or tasks, our experiments reveal that a policy trained on one backbone or domain may degrade when directly transferred to a very different setting. We leave exploring the better cross-domain generalization ability of RSD as future work, making it a \textbf{\textit{tuning-free}} plugin for all ranking tasks. 
Secondly, the up-to-down paradigm assumes a modest number of candidates (e.g., $K=25$). As $K$ grows, both the Transformer-based relevance network’s cost and the LLM encoding cost grow quadratically, potentially limiting RSD’s applicability to extremely large item corpora \cite{joglekar2020neural}.

\newpage
\bibliographystyle{plain}
\bibliography{neurips_2025}
\appendix

\section{Prompt for LLMs in IR and RS tasks}\label{sec:prompt}

\begin{tcolorbox}[breakable, colback=yellow!10, colframe=black!75!white, title=Prompt for LLMs in IR task.]
\textit{{\color{red} System}: You are RankLLM, an intelligent assistant that can rank passages based on their relevancy to the query.
\newline
{\color{red} User}: I will provide you with 20 passages, each indicated by an identifier |?|. Rank the passages based on their relevance to the search query: {\color{blue}[can you use paints on leather]}.
\newline
\newline {\color{magenta}|A| In order to remove these and allow the paint to bind, apply rubbing alcohol to a washcloth and thoroughly scrub the area of leather that will be painted. This will strip away the majority of the protective layers.}
\newline {\color{magenta}...}
\newline {\color{magenta}|T| Most solutions used to remove paint will also remove the pigment from your leather furniture. You can try water and mild soap if the paint is water based. If you're feeling confident, wait until the paint is completely dry and use rubbing alcohol or non-oil nail polish with a cotton swab.mmediately blot any spilled drinks to keep the liquid from soaking into the leather. To remove water or other clear drinks, roll a piece of white bread into a ball and blot the stain with it. Dark drinks will need specific care methods.}
\newline
\newline Search Query: {\color{blue}[can you use paints on leather]}.
\newline Rank the 20 passages above based on their relevance to the search query. All the passages should be included and listed using identifiers, in descending order of relevance. The output format should be formulated with |?|>|?| , e.g., |C| > |B| > |A| > |D|, Only respond with the ranking results, do not say any word or explain.}
\end{tcolorbox}

\begin{tcolorbox}[breakable, colback=yellow!10, colframe=black!75!white, title=Prompt for LLMs in RS task.]
\textit{{\color{red} System}: You are a Recommender expert, an intelligent assistant that can rank candidate items (e.g., movies) by analyzing user history behaviors.
\newline
{\color{red} User}: I will first provide you with user's historical behaviors (watching movies). The title of this user watched movie are  {\color{blue}<Princess Bride, The (1987)>, ..., <Dick Tracy (1990)>}. Now, I will provide you with 25 movies, each indicated by an identifier |?|. Please rank these movies based on user preference and her historical behaviors:
\newline\newline {\color{magenta}|A| Honey, I Shrunk the Kids (1989)}
\newline {\color{magenta}...}
\newline {\color{magenta}|Y| Perez Family, The (1995)}
\newline\newline Rank these 25 candidate movies above based on user preference and her historical behaviors. All the items should be included and listed using identifiers, in descending order of relevance. The output format should be formulated with |?|>|?|, e.g., |C| > |B| > |A| > |D|, Only respond with the ranking results, do not say any words or explain.}
\end{tcolorbox}

\setcounter{theorem}{0} 
\section{Detailed proof of theorems}\label{sec:proof}
\begin{theorem} \textbf{(Monotonicity.) }
For $i^*_T<K$, we have $i^*_t<i^*_{t+1}$ for all $t<T$, where {\small $i_t^* = \max\{ i | \forall j \in [0, i],\ \sigma_{t-1}[j] = \arg\max_d P_\text{llm}(d | x^{0}, \sigma_{t-1}[:j]) \}$} denotes the longest prefix of $\sigma_{t-1}$ that is consistent with greedy decoding from model conditioned on the prompt.
\end{theorem}
\begin{proof} With $i^*_T<K$ and $t+1<T$, we have
    {\small\begin{align*}
i_{t+1}^* &= \max\{ i | \forall j \in [0, i],\ \sigma_{t}[j] = \arg\max_d P_\text{llm}(d | x^{0}, \sigma_{t}[:j]) \}\\
&\overset{(a)}=\max\{ i | \forall j \in [i_{t+1}^*, i],\ \sigma_{t}[j] = \arg\max_d P_\text{llm}(d | x^{0}, \sigma_{t}[:j]) \}\\
&\overset{(b)}= \max \left\{i_{t}^*+1, \max\{ i | \forall j \in [i_{t+2}^*, i],\ \sigma_{t}[j] = \arg\max_d P_\text{llm}(d | x^{0}, \sigma_{t}[:j]) \}\right\}\\
&> i^*_t
\end{align*}}

\noindent where step (a) is established because for {\small $j\le i_{t}^*$ we have $\sigma_{t}[j]=\sigma_{t-1}[j] = \arg\max_d P_\text{llm}(d | x^{0}, \sigma_{t-1}[:j])=\arg\max_d P_\text{llm}(d | x^{0}, \sigma_{t}[:j])$} according to case $i\le i^*_t$ of \autoref{def_sigma}. Step (b) is established because {\small $\sigma_{t}[i_{t}^*+1] \overset{\text{(I)}}=\arg\max_d P_\text{llm}(d | x^{0}, \sigma_{t-1}[:i_{t}^*]) \overset{\text{(II)}}= \arg\max_d P_\text{llm}(d | x^{0}, \sigma_{t}[:i_{t}^*])$} according to case $i= i^*_t+1$ for (I) and $i\le i^*_t$ for (II), respectively.
\end{proof}

\begin{theorem}
 If we have $\pi_{\theta}/\pi_{\theta_{old}}\to 1$ (i.e., $|\pi_{\theta}/\pi_{\theta_{old}}- 1|<\epsilon$) and $\mathcal{G}\sim\pi_{\theta_{old}}$, we have our RPO objective is equivalent to GRPO objective (Equation (3) in \cite{shao2024deepseekmath}) in the first term w.r.t. the model parameter $\theta$, that is,
 {\small 
\begin{equation*}
    \mathcal{L}_{RPO-1} \approx \mathbb{E}_{Q,\pi_{\theta_{old}}}\frac{1}{G}\sum_{i=1}^{G}\frac{1}{T}\sum_{t=1}^T\{\min\left[\frac{\pi_\theta(\sigma_t^{(i)}|q,s^{(i)}_{<t})}{\pi_{\theta_{old}}(\sigma_t^{(i)}|q,s^{(i)}_{<t})}\hat{A}_i,\text{clip}\left(\frac{\pi_\theta(\sigma_t^{(i)}|q,s^{(i)}_{<t})}{\pi_{\theta_{old}}(\sigma_t^{(i)}|q,s^{(i)}_{<t})},1-\epsilon,1+\epsilon \right)\hat{A}_i\right] + C
\end{equation*}}

\noindent where $C$ is a $\theta$-independent constant.
\end{theorem}
\begin{proof} 
\begin{align*}
    \mathcal{L}_{RPO-1} &= \mathbb{E}_{Q,\pi_{\theta_{old}}}\frac{1}{G}\sum_{i=1}^{G}\frac{1}{T}\sum_{t=1}^T\hat{A}_i\log\pi_\theta(\sigma_t^{(i)}|q,s^{(i)}_{<t})\\
    &= \mathbb{E}_{Q,\pi_{\theta_{old}}}\frac{1}{G}\sum_{i=1}^{G}\frac{1}{T}\sum_{t=1}^T\hat{A}_i[\log\pi_{\theta_{old}}(\sigma_t^{(i)}|q,s_{<t}^{(i)})+\log\frac{\pi_{\theta}(\sigma_t^{(i)}|q,s_{<t}^{(i)})}{\pi_{\theta_{old}}(\sigma_t^{(i)}|q,s_{<t}^{(i)})}]
\end{align*}
According to $|\pi_{\theta}/\pi_{\theta_{old}}- 1|\to 0$, we have the following formulation based Taylor expansion:
\begin{align}
    \mathcal{L}_{RPO-1} &= \mathbb{E}_{Q,\pi_{\theta_{old}}}[\frac{1}{G}\sum_{i=1}^{G}\frac{1}{T}\sum_{t=1}^T\hat{A}_i\cdot\log[\frac{\pi_{\theta}(\sigma_t^{(i)}|q,s_{<t}^{(i)})}{\pi_{\theta_{old}}(\sigma_t^{(i)}|q,s_{<t}^{(i)})}-1+1]] +C_1 \notag \\
    &= \mathbb{E}_{Q,\pi_{\theta_{old}}}[\frac{1}{G}\sum_{i=1}^{G}\frac{1}{T}\sum_{t=1}^T\hat{A}_i\cdot\frac{\pi_{\theta}(\sigma_t^{(i)}|q,s_{<t}^{(i)})}{\pi_{\theta_{old}}(\sigma_t^{(i)}|q,s^{(i)}_{<t})}-1 +O(\frac{\pi_{\theta}(\sigma_t^{(i)}|q,s_{<t}^{(i)})}{\pi_{\theta_{old}}(\sigma_t^{(i)}|q,s_{<t}^{(i)})}-1)] +C\notag  \\
    &\approx \mathbb{E}_{Q,\pi_{\theta_{old}}}[\frac{1}{G}\sum_{i=1}^{G}\frac{1}{T}\sum_{t=1}^T\hat{A}_i\cdot[\frac{\pi_{\theta}(\sigma_t^{(i)}|q,s_{<t}^{(i)})}{\pi_{\theta_{old}}(\sigma_t^{(i)}|q,s^{(i)}_{<t})}-1]] +C_1 \notag \\
    &=  \mathbb{E}_{Q,\pi_{\theta_{old}}}[\frac{1}{G}\sum_{i=1}^{G}\frac{1}{T}\sum_{t=1}^T\hat{A}_i\cdot\frac{\pi_{\theta}(\sigma_t^{(i)}|q,s_{<t}^{(i)})}{\pi_{\theta_{old}}(\sigma_t^{(i)}|q,s^{(i)}_{<t})}] +C_2 \label{eq:rpo}
\end{align}
For $|\pi_{\theta}/\pi_{\theta_{old}}- 1|\to 0$, we have $-\epsilon<\pi_{\theta}/\pi_{\theta_{old}}- 1<\epsilon$ with specific $\epsilon>0$. Therefore, $\text{clip}\left(\frac{\pi_\theta(\sigma_t^{(i)}|q,s^{(i)}_{<t})}{\pi_{\theta_{old}}(\sigma_t^{(i)}|q,s^{(i)}_{<t})},1-\epsilon,1+\epsilon \right) =  \frac{\pi_\theta(\sigma_t^{(i)}|q,s_{<t}^{(i)})}{\pi_{\theta_{old}}(\sigma_t^{(i)}|q,s^{(i)}_{<t})}$ and $\min(x,x)=x$ are always established. Applying these formulations into \autoref{eq:rpo}, the conclusion in the theorem is easily established.

\end{proof}

\begin{lemma} \label{lem:1}
For the group-mean advantage $\hat{A}_i^\text{group}$ and reference advantage $\hat{A}_i^\text{ref}$, we have $\mathbb E(\hat{A}_i^\text{group})=0$ and $\mathbb E(\hat{A}_i^\text{ref})=0$ by assigning $\pi_{\theta_{ref}} \gets \pi_{\theta}$ in the training process.
\end{lemma}
\begin{proof}
    For group-mean advantage, we have $\mathbb E(\hat{A}_i^\text{group})=\mathbb E(R)-\frac{1}{G}\cdot G\cdot\mathbb E(R)=0$. For reference advantage, we have $\mathbb E(\hat{A}_i^\text{ref})=\mathbb E(R)-\mathbb E(R)=0$
\end{proof}

\begin{lemma}[Zero-Mean Score Function]
\label{lem:2}
Let $\{\pi_{\theta}(\sigma_t|q,s_{<t})\bigr\}_{q\sim Q,\sigma_q\sim\pi_{\theta}(\cdot|q)}$ be a family of probability, continuously differentiable with respect to the parameter
vector $\theta\in\mathbb R^{d}$.  
Define the \emph{score function}
\[
g(\sigma_t|q,s_{<t})\;:=\;\nabla_{\theta}\log\pi_{\theta}(\sigma_t|q,s_{<t})
           \;=\;
           \Bigr(
             \nabla_{\theta_{1}}\log\pi_{\theta}(\sigma_t|q,s_{<t}),
             \dots,
             \nabla_{\theta_{d}}\log\pi_{\theta}(\sigma_t|q,s_{<t})
           \Bigr).
\]
Then for every coordinate $k\in\{1,\dots,d\}$, we have $\mathbb{E}[g_{k}(\sigma_t|q,s_{<t})\bigr] = 0$.

\end{lemma}

\begin{proof}
As $Q(\cdot)$ and $\pi_{\theta}(\cdot)$ is a probability distribution, we have 

\[\sum\nolimits_{q\sim\mathcal{Q}}\sum\nolimits_{\sigma\sim\pi_{\theta}(\cdot|q)}Q(q)\cdot\pi_{\theta}(\sigma_t|q,s_{<t})=1\]

Differentiate both sides with respect to the $k$-th parameter~$\theta_k$:

\begin{align*}
0 &=\sum_{q\sim\mathcal{Q}}\sum_{\sigma\sim\pi_{\theta}(\cdot|q)}{Q(q)\cdot\nabla_{\theta_{k}}\pi_{\theta}(\sigma_t|q,s_{<t})}\\
&=\sum_{q\sim\mathcal{Q}}\sum_{\sigma\sim\pi_{\theta}(\cdot|q)}Q(q)\cdot\pi_{\theta}(\tau)\,
  {\nabla_{\theta_{k}}\log\pi_{\theta}(\sigma_t|q,s_{<t})}\\
&=\mathbb{E}_{q\sim\mathcal{Q},\sigma\sim\pi_{\theta}(\cdot|q)}[g_{k}(\sigma_t|q,s_{<t})\bigr].
\end{align*}

\end{proof}

\begin{theorem}
By  assigning $\pi_{\theta_{ref}} \gets \pi_{\theta}$ in the training process, we have both updating the gradient in $\mathcal{L}_{RPO-1}$ w.r.t. group average advantage $\hat{A}_i^\text{group}$ and reference advantage $\hat{A}_i^\text{ref}$ is unbiased, i.e.,
\begin{equation*}
 \mathbb E[(R_i-B)\cdot \nabla_{\theta_{d}}\log\pi_{\theta}(\sigma_t|q,s_{<t})]=\mathbb E[R\cdot\nabla_{\theta_{d}}\log\pi_{\theta}(\sigma_t|q,s_{<t})]
\end{equation*}
where $B=\hat{\mu}_{-i}$ or $B = R_\text{ref}$. Suppose the return of $R_i= \mu + \varepsilon_i$ and $R_{\text{ref}} = \mu + \delta$ can be decomposed into query-level reward (related to the difficulty of the question), i.e., $\mu\sim \mathcal N\!(0,\sigma_b^{2}\bigr)$, and and temperature noise (related to the random sampling), i.e., $\delta\sim \mathcal N\!(0,\sigma_\delta^{2}\bigr)$ and $\varepsilon_i \stackrel{\text{i.i.d.}}{\sim} \mathcal N\!(0,\sigma_w^{2}\bigr)$ with $\varepsilon_i\perp\delta$, we have their variance 
\begin{equation*}
    Var[(R-R_\text{ref})\cdot \nabla_{\theta_{d}}\log\pi_{\theta}(\sigma_t|q,s_{<t})] < Var[(R-\hat{\mu}_{-i})\cdot \nabla_{\theta_{d}}\log\pi_{\theta}(\sigma_t|q,s_{<t})]
\end{equation*}
 if and only if  $ \sigma_\delta^{2} < {\sigma_w^{2} }/{(G-1)}$.

\end{theorem}

\begin{proof}
\textbf{Unbiasedness.}  Suppose $B = \hat{\mu}_{-i}$ and $B = R_{\text{ref}}$, we have their estimation to $\mathbb{E}[g_k^{(i)}\cdot(R_i-B_i)] = \mathbb{E}[g_k^{(i)}\cdot R_i]$. Specifically, the group trajectories $\mathcal{G}_q = \{(R_i, g_k(\sigma_t^{(i)}|q,s_{<t})) | i=1,\cdots,G\} \sim \pi_\theta$ are i.i.d., satisfying $\hat{\mu}_{-i} \perp g^{(i)}_k(\sigma_t|q,s_{<t})$. Therefore, 
\[
\mathbb{E}[g_k^{(i)}(R_i - \hat{\mu}_{-i})] = \mathbb{E}[g_k^{(i)} R_i] + \mathbb{E}[g_k^{(i)}] \cdot \mathbb{E}[\hat{\mu}_{-i}] = \mathbb{E}[g_k^{(i)} R_i]
\]
according to Lemma~\autoref{lem:2}. Sampling from the reference model $\pi_{\theta_{\text{ref}}}$ is independent from sampling from the current policy  $\pi_{\theta}$, satisfying $R_{\text{ref}} \perp g_k$. Therefore, 
\[
\mathbb{E}[g_k^{(i)}(R_i - R_{\text{ref}})] = \mathbb{E}[g_k^{(i)}\cdot R_i] + \mathbb{E}[g_k^{(i)}] \cdot \mathbb{E}[R_{\text{ref}}] = \mathbb{E}[g^{(i)}_k R_i].
\]
In summary, no matter whether $B = \hat{\mu}_{-i}$ or $B = R_{\text{ref}}$, their estimations towards $\mathbb{E}[g^{(i)}_k R_i]$ are unbiased.

\textbf{Variance.}
Suppose we have two variables $X$ and $Y$ satisfying $X\bot Y$, we have 
\begin{equation*}
    Var[XY] = \mathbb E[X^2]\mathbb E[Y^2] - (\mathbb E[XY])^2.
\end{equation*}
Assigning $X=R_i-B$ and $Y=g_k^{(i)}$, we have 
\begin{align*}
        Var[(R_i-B)g_k^{(i)}] &{=} \mathbb E[(g_k^{(i)})^2]\cdot \mathbb E[(R_i-B)^2] - (\mathbb E[g_k^{(i)}(R_i-B)])^2\\
        & = \sigma^2_g\cdot \mathbb E[(R_i-B)^2] - (\mathbb E[g_k^{(i)}\cdot(R_i-B)])^2\\
        & \overset{(a)}= \sigma^2_g\cdot\{Var[R_i-B]+(\mathbb E [R_i-B])^2\} -(\mathbb E(g_k^{(i)}\cdot R_i)-\mathbb E(g_k^{(i)})\mathbb E(R_i))^2 \\
        & \overset{(b)}= \sigma^2_g\cdot Var[R_i-B] -(\mathbb E(g_k^{(i)}\cdot R_i)-\mathbb E(g_k^{(i)})\mathbb E(R_i))^2 \\
        &\overset{(c)}= \sigma^2_g\cdot Var[R_i-B] -(\mathbb E(g_k^{(i)}\cdot R_i))^2
\end{align*}
where step (a) is established because $R_\text{ref}\bot(R_i,g_k^{(i)})$, step (b) is established because $\mathbb E [R_i-B]=0$ according to Lemma \autoref{lem:1}, step (c) is established  because $\mathbb E [g_k^{(i)}]=0$ according to Lemma \autoref{lem:2}. Therefore, comparing their gradient variance is equivalent to comparing the variance of their advantages, i.e.,
\begin{equation*}
    Var[(R-R_\text{ref})g_k^{(i)}] - Var[(R-\hat{\mu}_{-i})g_k^{(i)}] \Leftrightarrow Var[R_i-R_\text{ref}] - Var[R_i-\hat{\mu}_{-i}]
\end{equation*}
For these advantage variances, suppose the return of $R_i= \mu + \varepsilon_i$ and $R_{\text{ref}} = \mu + \delta$ can be decomposed into query-level reward (related to the difficulty of the question), i.e., $\mu\sim \mathcal N\!(0,\sigma_b^{2}\bigr)$, and temperature noise (related to the random sampling), i.e., $\delta\sim \mathcal N\!(0,\sigma_\delta^{2}\bigr)$ and $\varepsilon_i \stackrel{\text{i.i.d.}}{\sim} \mathcal N\!(0,\sigma_w^{2}\bigr)$ with $\varepsilon_i\perp\delta$, they can be easily formulated as follows:
\begin{align*}
    Var[R_i-R_\text{ref}]  = Var[\varepsilon_i-\delta] = \sigma_w^{2} + \sigma_\delta^{2}
\end{align*}
\begin{align*}
Var[R_{pi} - \widehat{\mu}_{-i}]
  &= Var[(\mu + \varepsilon_{i}\bigr)
     -
     (
       \mu
       + \frac{1}{G-1}\sum_{j\neq i}\varepsilon_{j}
     \Bigr)]\\
  &= Var[\varepsilon_{i}
     -
     \frac{1}{G-1}\sum_{j\neq i}\varepsilon_{j}]= \sigma_w^{2}\frac{G}{G-1}
\end{align*}
Therefore, $Var[R-R_\text{ref}] - Var[R-\hat{\mu}_{-i}]<0$ if and only if
\begin{equation*}
     \sigma_\delta^{2} < \frac{\sigma_w^{2} }{G-1}
\end{equation*}
\end{proof}

\section{Evaluation metrics for rank similarity}\label{ranking_measure}

\subsection{Kendall-tau}
Kendall’s tau measures the ordinal association between two rankings based on concordant and discordant pairs:
\[
\tau = \frac{C - D}{\tfrac{1}{2}K(K - 1)}
\]
where $C$ is the number of concordant pairs and $D$ is the number of discordant pairs, defined as:
\[
C = \sum_{1 \le i < j \le n} \mathbf{1}\left[ (r_1(i) - r_1(j))(r_2(i) - r_2(j)) > 0 \right]
\]
\[
D = \sum_{1 \le i < j \le n} \mathbf{1}\left[ (r_1(i) - r_1(j))(r_2(i) - r_2(j)) < 0 \right]
\]
where $r_1(i)$ and $r_2(i)$ denote the rank of item $i$ in rankings $r_1$ and $r_2$, respectively. The indicator function $\mathbf{1}[\cdot]$ equals 1 when the condition is true, and 0 otherwise.

The value range is $\tau \in [-1, 1]$, where larger values (closer to 1) indicate higher similarity.

\subsection{Spearman-rho}
Spearman’s rho evaluates rank correlation by computing the squared differences of ranks:
\[
\rho = 1 - \frac{6 \sum_{i=1}^{K}(r_1(i) - r_2(i))^2}{K(K^2 - 1)}
\]

The value range is $\rho \in [-1, 1]$, where larger values (closer to 1) indicate higher similarity.

\subsection{Footrule distance}
Footrule distance is the sum of absolute rank differences:
\[
D_{\text{Footrule}}(r_1, r_2) = \sum_{i=1}^{K} |r_1(i) - r_2(i)|
\]
The maximum value depends on the permutation, but for full rankings, i.e., $D_{\text{Footrule}} \in [0, \tfrac{K^2}{2}]$, where smaller values indicate higher similarity.

\subsection{Kemeny distance}
Kemeny distance counts the number of discordant pairs between two rankings. For rankings $r_1$ and $r_2$, it is defined as:
\[
D_{\text{Kemeny}}(r_1, r_2) = \sum_{1 \le i < j \le K} \mathbf{1} \left[ (r_1(i) < r_1(j)) \ne (r_2(i) < r_2(j)) \right]
\]
The value range is $D_{\text{Kemeny}} \in [0, \tfrac{K(K-1)}{2}]$, where smaller values indicate higher similarity.

\section{Introduction and Adaption of Baseline Methods}\label{sec_baseline}

For baseline methods, we adopt the following state-of-the-art methods with the STD strategy, i.e., Rankzephyr and FIRST, and methods with the SD strategy, i.e., Speculative Decoding Method (SDM) \cite{leviathan2023fast}, Blockwise \cite{stern2018blockwise}, Medusa \cite{cai2024medusa}, and Parallel Decoding Method (PDM) \cite{santilli2023accelerating}, as well as a hybrid method combining FIRST and the SD strategy, i.e., FIRST + Speculative Decoding. As there is a difference between our task and those of these methods, necessary modifications need to be adopted for fair comparison.

\begin{itemize}
\item \textbf{Rankzephyr} is trained using a standard language modeling objective, with the ranking generated by GPT-4 as the target ranking. For fair comparison, we adopt the ranking generated by the target LLM as the target ranking, and we rank indicated by the logits of its first token position matches that of its fully generated ranking. 

\item \textbf{FIRST} adopts a listwise LLM reranking approach that utilizes a learning-to-rank loss and SFT loss for fine-tuning. For fair comparison, we adopt the ranking generated by the target LLM as the ground truth ranking.

\item \textbf{Speculative Decoding Method (SDM)} proposes a SD strategy that can accelerate decoding from self-regressive models without any change to the model architecture. For fair comparison, we adopt the item ranking indicated by the logits at the last rejection position as the output of the approximation model. To rank all candidate items within a limited budget, we directly rank the unaccepted items based on the logits at the last rejection position when the budget is reached.

\item \textbf{Blockwise} proposes a blockwise parallel decoding scheme to make predictions for multiple time steps in parallel. To enable it to rank all candidate items within a limited budget, we set the number of prediction heads equal to the number of candidate items for drafting and adopt a Top-5 selection strategy for verification. Once the budget is reached, we directly rank the unaccepted items based on their prediction heads.

\item \textbf{Medusa} adopts a tree-based attention mechanism to construct multiple candidate continuations and verifies them simultaneously at each decoding step. To enable it to rank all candidate items within a limited budget, we set the number of prediction heads equal to the number of candidate items. Once the budget is reached, we directly rank the unaccepted items based on their prediction heads. For fair comparison, we adopt Medusa-1, directly fine-tuned on top of a frozen backbone LLM, verifying candidate continuations with a tree of size [1, 3, 3, 3] for the remaining items, where the cumulative number of new tokens is twice that of ours.

\item \textbf{Parallel Decoding Method (PDM)} proposes a parallel formulation leveraging Jacobi and Gauss-Seidel fixed-point iteration methods for fast inference. For fair comparison, we return the current decoded item ranking when the budget is reached.
\end{itemize}

\section{The additional tables and figures in experiments.}

\renewcommand{\arraystretch}{1.6}
\begin{table}[H]
  \centering
  \fontsize{8}{8}\selectfont
  \caption{{\small Performance of different methods on \textbf{IR tasks} with the LLM backbone \textbf{Qwen2.5-7B-Instruct}. * indicates statistically significant improvement of RSD to baselines on t-test (p < 0.05).}}
\begin{tabular}{||c||c|c|c|c||c|c|c|c||}\midrule
    Dataset              & \multicolumn{4}{c||}{MS MARCO}                                        & \multicolumn{4}{c||}{Quora}                                           \\\midrule
Method               & KT $\uparrow$             & SR $\uparrow$             & FD  $\downarrow$           & KD    $\downarrow$         & KT  $\uparrow$            & SR     $\uparrow$         & FD   $\downarrow$          & KD  $\downarrow$           \\\midrule
Random               & 0.003          & 0.007          & 132.30         & 94.72          & -0.012         & -0.016         & 209.69         & 151.77         \\
Speculative Decoding & 0.548          & 0.680          & 63.85          & 42.97          & 0.454          & 0.588          & 116.88         & 81.95          \\
Blockwise            & 0.6179	&0.7438	&55.33	&36.30          & 0.503          & 0.630          & 107.17         & 74.53          \\
Mesuda               & 0.624          & 0.751          & 53.55          & 35.74          & 0.521          & 0.653          & 102.81         & 71.82          \\
Parallel Decoding    & 0.504          & 0.630          & 106.17         & 74.45          & 0.504          & 0.630          & 106.17         & 74.45          \\
Rankzephyr           & 0.351          & 0.471          & 87.97          & 61.62          & 0.262          & 0.370          & 153.68         & 110.68         \\
FIRST                & 0.166          & 0.206          & { 112.46}   & 79.18          & 0.109          & 0.161          & 184.52         & 133.58         \\
FIRST + SD           & 0.168          & 0.181          & 112.69         & 79.05          & 0.062          & 0.104          & 192.50         & 140.66         \\
RSD (ours)                 & \textbf{0.700*} & \textbf{0.822*} & \textbf{43.41*} & \textbf{28.54*} & \textbf{0.550*} & \textbf{0.686*} & \textbf{97.02*} & \textbf{67.51} \\ \midrule
Imprv.               & 3.54\%         & 4.09\%         & 9.54\%         & 7.38\%         & 5.52\%         & 5.11\%         & 5.63\%         & 6.01\%         \\ \midrule
    \end{tabular}\label{appendix_table1}
\end{table}

\renewcommand{\arraystretch}{1.6}
\begin{table}[H]
  \centering
  \fontsize{8}{8}\selectfont
  \caption{{\small Performance of different methods on \textbf{RS tasks} with the LLM backbone \textbf{Llama-3.2-3B-Instruct}. * indicates statistically significant improvement of RSD to baselines on t-test (p < 0.05).}}
\begin{tabular}{||c||c|c|c|c||c|c|c|c||}\midrule
    Dataset              & \multicolumn{4}{c||}{ML-1M}                                        & \multicolumn{4}{c||}{Amazon-Games}                                           \\\midrule
Method               & KT $\uparrow$             & SR $\uparrow$             & FD  $\downarrow$           & KD    $\downarrow$         & KT  $\uparrow$            & SR     $\uparrow$         & FD   $\downarrow$          & KD  $\downarrow$           \\\midrule
Random               & 0.0041          & 0.0070          & 206.53         & 149.38         & 0.0080          & 0.0110          & 206.38         & 148.80         \\
Speculative Decoding & 0.3571          & 0.4608          & 136.44         & 96.44          & 0.3571          & 0.4608          & 136.44         & 96.44          \\
Blockwise            & 0.5468          & 0.6563          & 99.16          & 67.98          & 0.5412          & 0.6464          & 97.90          & 68.82          \\
Mesuda               & 0.3907          & 0.4757          & 131.52         & 91.39          & 0.4815          & 0.5823          & 109.66         & 77.78          \\
Parallel Decoding    & 0.4863          & 0.6015          & 113.22         & 77.06          & 0.5168          & 0.6291          & 108.34         & 72.48          \\
Rankzephyr           & 0.0491          & 0.0740          & 199.00         & 142.63         & -0.0628         & -0.0732         & 222.67         & 159.42         \\
FIRST                & 0.0729          & 0.1083          & 195.23         & 139.07         & 0.1632          & 0.2223          & 176.81         & 125.52         \\
FIRST + SD           & 0.3442          & 0.4468          & 139.23         & 98.37          & 0.3931          & 0.4996          & 127.81         & 91.03          \\
RSD (ours)                 & \textbf{0.6426*} & \textbf{0.7659*} & \textbf{78.57*} & \textbf{53.61*} & \textbf{0.6901*} & \textbf{0.8080*} & \textbf{68.38*} & \textbf{46.48*} \\\midrule
Imprv.               & 17.52\%         & 16.70\%         & 20.77\%        & 21.14\%        & 27.52\%         & 25.01\%         & 30.15\%        & 32.46\%      \\ \midrule
    \end{tabular}\label{appendix_table2}
\end{table}

\renewcommand{\arraystretch}{1.6}
\begin{table}[H]
  \centering
  \fontsize{8}{8}\selectfont
  \caption{{\small Ablation studies on {IR and RS tasks} with the different LLM backbones. }}
  \begin{tabular}{@{}|c|c|cccc|cccc|@{}}
\midrule\multirow{2}{*}{Dataset}         & Backbone & \multicolumn{4}{c|}{Llama-3.2-3B}                                      & \multicolumn{4}{c|}{Qwen2.5-7B}                                        \\
& Method   & KT $\uparrow$             & SR $\uparrow$             & FD  $\downarrow$           & KD    $\downarrow$         & KT  $\uparrow$            & SR     $\uparrow$         & FD   $\downarrow$          & KD  $\downarrow$           \\\midrule
\multirow{6}{*}{MS MARCO}     & STD      & 0.3042          & 0.4060          & 96.256          & 66.104          & 0.3529          & 0.4757          & 88.040          & 61.470          \\
& GSD      & 0.6360          & 0.7523          & 49.516          & 34.582          & 0.6337          & 0.7543          & 50.100          & 34.794          \\
& w/o RPO  & 0.6797          & 0.8088          & 46.048          & 30.426          & 0.6823          & 0.8088          & 45.420          & 30.180          \\
& w/o  LR  & 0.6907          & 0.8180          & 44.804          & 29.388          & 0.6789          & 0.8033          & 45.842          & 30.505          \\
& w/o RA   & 0.6065          & 0.7452          & 55.104          & 37.386          & 0.5343          & 0.6688          & 64.004          & 44.242          \\
& RSD      & \textbf{0.7169} & \textbf{0.8371} & \textbf{40.964} & \textbf{26.896} & \textbf{0.6996} & \textbf{0.8217} & \textbf{43.408} & \textbf{28.536} \\\midrule\midrule
\multirow{6}{*}{Quora}        & STD      & 0.0722          & 0.1106          & 192.320         & 139.166         & 0.2825          & 0.3914          & 152.868         & 107.618         \\
& GSD      & 0.5586          & 0.6821          & 94.184          & 66.214          & 0.5301          & 0.6586          & 99.068          & 70.490          \\
& w/o RPO  & 0.6324          & 0.7555          & 80.896          & 55.144          & 0.6283          & 0.7616          & 82.332          & 55.750          \\
& w/o  LRK  & 0.6274          & 0.7493          & 81.904          & 55.884          & 0.6292          & 0.7622          & 82.388          & 55.624          \\
& w/o RA   & 0.5376          & 0.6779          & 98.832          & 69.358          & 0.5214          & 0.6608          & 105.304         & 71.788          \\
& RSD      & \textbf{0.6426} & \textbf{0.7659} & \textbf{78.568} & \textbf{53.614} & \textbf{0.6404} & \textbf{0.7716} & \textbf{80.376} & \textbf{53.946} \\\midrule\midrule
\multirow{6}{*}{ML-1M}        & STD      & 0.2401          & 0.3352          & 161.912         & 113.986         & 0.2427          & 0.3405          & 158.600         & 113.590         \\
& GSD      & 0.5684          & 0.6951          & 91.128          & 64.742          & 0.5259          & 0.6580          & 99.848          & 71.118          \\
& w/o RPO  & 0.6380          & 0.7739          & 79.708          & 54.304          & 0.5419          & 0.6791          & 98.444          & 68.718          \\
& w/o  LRK  & 0.6383          & 0.7739          & 79.840          & 54.260          & 0.5420          & 0.6784          & 98.724          & 68.700          \\
& w/o RA   & 0.5511          & 0.6914          & 97.012          & 67.328          & 0.4524          & 0.5852          & 115.816         & 82.138          \\
& RSD      & \textbf{0.6454} & \textbf{0.7824} & \textbf{78.808} & \textbf{53.186} & \textbf{0.5499} & \textbf{0.6860} & \textbf{97.016} & \textbf{67.510} \\\midrule\midrule
\multirow{6}{*}{Amazon-Games} & STD      & 0.0889          & 0.1333          & 190.636         & 136.666         & 0.2328          & 0.3273          & 160.436         & 115.084         \\
& GSD      & 0.5839          & 0.7064          & 88.200          & 62.412          & 0.5787          & 0.7073          & 88.808          & 63.188          \\
& w/o RPO  & 0.6766          & 0.7991          & 70.160          & 48.516          & 0.6571          & 0.7976          & 76.364          & 51.436          \\
& w/o  LRK  & 0.6880          & 0.8076          & 67.992          & 46.798          & 0.6515          & 0.7930          & 77.552          & 52.268          \\
& w/o RA   & 0.6479          & 0.7805          & 76.204          & 52.818          & 0.5551          & 0.6953          & 97.528          & 66.734          \\
& RSD      & \textbf{0.6901} & \textbf{0.8080} & \textbf{68.380} & \textbf{46.480} & \textbf{0.6832} & \textbf{0.8178} & \textbf{71.040} & 47.524         
\\ \midrule
    \end{tabular}\label{appendix_ablation}
\end{table}


\end{document}